\documentclass[journal]{IEEEtran} 
\usepackage{amsmath,amsfonts,bm}
\usepackage{amsthm,amsmath,amssymb}
\usepackage{eucal}
\usepackage{xifthen}
\usepackage{mathtools}
\usepackage{enumerate}
\usepackage{microtype}
\usepackage{xspace}
\usepackage{bm}
\usepackage{cite}
\usepackage{fancyhdr}
\usepackage{lastpage}
\usepackage[small]{caption}
\usepackage{xcolor}
\usepackage{algorithm}
\usepackage{algpseudocode}
\allowdisplaybreaks

\def\eqref#1{Eq-(\ref{#1})}

\newcommand{\nn}{\nonumber}

\newtheorem{remark}{Remark}
\newtheorem{theorem}{Theorem}
\newtheorem{corollary}{Corollary}

\def\1{\bm{1}}

\allowdisplaybreaks
\newcommand{\st}{{\rm s.t.}}
\DeclareMathAlphabet{\mathcal}{OMS}{cmsy}{m}{n}

\def\vzero{{\bm{0}}}
\def\vone{{\bm{1}}}

\def\va{{\bm{a}}}

\def\vu{{\bm{u}}}
\def\vv{{\bm{v}}}
\def\vw{{\bm{w}}}
\def\vx{{\bm{x}}}

\def\vz{{\bm{z}}}

\def\mA{{\bm{A}}}
\def\mB{{\bm{B}}}

\def\mF{{\bm{F}}}

\def\mI{{\bm{I}}}
\def\mJ{{\bm{J}}}

\def\mM{{\bm{M}}}

\def\mP{{\bm{P}}}

\def\mR{{\bm{R}}}

\def\mT{{\bm{T}}}
\def\mU{{\bm{U}}}
\def\mV{{\bm{V}}}
\def\mW{{\bm{W}}}

\def\mLambda{{\bm{\Lambda}}}
\def\mSigma{{\bm{\Sigma}}}

\DeclareMathAlphabet{\mathsfit}{\encodingdefault}{\sfdefault}{m}{sl}
\SetMathAlphabet{\mathsfit}{bold}{\encodingdefault}{\sfdefault}{bx}{n}

\def\gA{{\mathcal{A}}}

\def\gD{{\mathcal{D}}}

\def\gJ{{\mathcal{J}}}

\def\gN{{\mathcal{N}}}
\def\gO{{\mathcal{O}}}

\def\gS{{\mathcal{S}}}

\def\gX{{\mathcal{X}}}

\newcommand{\E}{\mathbb{E}}

\newcommand{\R}{\mathbb{R}}

\newcommand{\KL}{D_{\mathrm{KL}}}

\newcommand{\Norm}[1]{\|#1\|}
\newcommand{\Abs}[1]{|#1|}

\DeclareMathOperator{\Tr}{Tr}

\usepackage{hyperref}
\usepackage{url}
\usepackage{amssymb}
\usepackage{amsmath}

\title{Towards A Unified PAC-Bayesian Framework for Norm-based Generalization Bounds}

\author{Xinping Yi, Gaojie Jin, Xiaowei Huang, and Shi Jin 
\thanks{X. Yi and S. Jin are with the School of Information Science and Engineering, Southeast University, Nanjing, China. Email: \texttt{\{xyi,jinshi\}@seu.edu.cn}.}
\thanks{G. Jin is with the Department of Computer Science, University of Exeter, Exeter, United Kingdom. Email: \texttt{gaojie.jin.kim@gmail.com}.}
\thanks{X. Huang is with the Department of Computer Science, University of Liverpool, Liverpool, United Kingdom. Email: \texttt{xiaowei.huang@
liverpool.ac.uk}.}
}

\begin{document}

\maketitle

\begin{abstract}
Understanding the generalization behavior of deep neural networks remains a fundamental challenge in modern statistical learning theory. Among existing approaches, PAC-Bayesian norm-based bounds have demonstrated particular promise due to their data-dependent nature and their ability to capture algorithmic and geometric properties of learned models. However, most existing results rely on isotropic Gaussian posteriors, heavy use of spectral-norm concentration for weight perturbations, and largely architecture-agnostic analyses, which together limit both the tightness and practical relevance of the resulting bounds. To address these limitations, in this work, we propose a unified framework for PAC-Bayesian norm-based generalization by reformulating the derivation of generalization bounds as a stochastic optimization problem over anisotropic Gaussian posteriors. The key to our approach is a sensitivity matrix that quantifies the network outputs with respect to structured weight perturbations, enabling the explicit incorporation of heterogeneous parameter sensitivities and architectural structures. By imposing different structural assumptions on this sensitivity matrix, we derive a family of generalization bounds that recover several existing PAC-Bayesian results as special cases, while yielding bounds that are comparable to or tighter than state-of-the-art approaches. Such a unified framework provides a principled and flexible way for geometry-/structure-aware and interpretable generalization analysis in deep learning.
\end{abstract}


\section{Introduction}
Lying at the heart of statistical learning theory, generalization remains a central challenge in understanding the empirical success of deep learning models, whose expressive power with over-parameterization often defies classical intuition. 
A generalization error bound quantitatively characterizes generalization performance by measuring, with high probability, the discrepancy between a model’s empirical performance on the training data and its expected performance on unseen test data drawn from the underlying data distribution.
Over the past decades, a rich collection of theoretical frameworks has been developed to derive the generalization error bounds, including VC dimension \cite{blumer1989learnability,vapnik2015uniform}, Rademacher complexity \cite{bartlett2002rademacher}, algorithmic stability \cite{bousquet2002stability,hardt2016train}, information-theoretic \cite{russo2016controlling,xu2017information} and compression-based \cite{arora2018stronger} approaches, as well as PAC-Bayesian methods \cite{mcallester1998some,mcallester1999pac,catoni2007pac}. 
Among these frameworks, PAC-Bayesian theory has emerged as one of the most promising and versatile tools for analyzing deep learning \cite{dziugaite2017computing,neyshabur2018pac,jin2020does,jin2022weight}, owing to its ability to yield non-vacuous, data-dependent bounds that naturally incorporate randomized predictors, prior knowledge, and algorithmic biases. 

The PAC-Bayesian framework combines probably approximately correct (PAC) learning and Bayesian statistics. The former focuses on providing theoretical guarantees about the generalization performance of a learning algorithm based on the amount of observed (cf. training) data, while the latter provides reasoning about uncertainty by incorporating prior beliefs and their posteriors updated with observed data using Bayes' theorem.
With respect to the PAC-Bayesian generalization, upper bounds on the expected generalization error of a learning algorithm are usually given as the sum of the empirical error on the training data and a complexity measure over the hypothesis space of the learning algorithm. These bounds are typically in the form of high probability, meaning that with high probability (e.g., at least $1-\delta$ with a small $\delta$), the true generalization error of the learned model will not exceed a certain threshold.
The complexity measure is usually expressed in terms of KL divergence, which measures the divergence between the distribution induced by the learned model (cf. the posterior) and a predetermined distribution irrelevant to models (cf. the prior). 
PAC-Bayesian theory has found applications in a vast range of machine learning tasks, including classification \cite{mcallester1998some,mcallester1999pac,mcallester2003simplified,langford2002pac}, regression \cite{alquier2016properties,germain2016pac,shalaeva2020improved,guo2025pacbayes}, and many others (see \cite{alquier2024user,guedj2019primer} and references therein). It provides a principled way to reason about generalization and uncertainty in learning algorithms, attracting increasingly more attention, especially to deep neural networks.

\subsection{Related Works}
PAC-Bayesian theory originates from the seminal works of McAllester \cite{mcallester1998some,mcallester1999pac} and Catoni \cite{catoni2007pac}, which established high-probability generalization bounds for randomized predictors in terms of the empirical risk and the Kullback-Leibler (KL) divergence between posterior and prior distributions over hypotheses. 
Over the past decades, there have been numerous follow-up works, e.g., \cite{langford2002pac,seeger2002pac,maurer2004note} to name a few, advancing this line of research. Subsequent developments extended the framework to data-dependent priors \cite{parrado2012pac,dziugaite2018data}, norm-based losses \cite{mcallester2003simplified,neyshabur2018pac}, and the adversarial settings \cite{viallard2021pac,xiao2023pac}, as well as those built close connections between PAC-Bayes, Bayesian inference, and information-theoretic generalization \cite{hellstrom2025generalization}. These advances positioned PAC-Bayes as a flexible framework capable of producing non-vacuous, algorithm-aware bounds even in highly over-parameterized regimes. 
A major line of research applies PAC-Bayesian theory to deep neural networks by carefully choosing priors and posteriors that lead to bounds expressed in terms of parameter norms \cite{neyshabur2015norm,neyshabur2018pac,dziugaite2017computing,bartlett2017spectrally}. By considering Gaussian or scale-invariant priors, these works derive generalization bounds depending on the Frobenius norm, path norm, or layerwise norms of network weights, often combined with Lipschitz properties of activations. Notably, PAC-Bayesian analyses have shown that flat minima and small perturbation sensitivity correspond to smaller KL terms, providing a theoretical explanation for empirical observations in deep learning. These norm-based PAC-Bayesian bounds bridge classical capacity control and modern optimization phenomena, yielding non-vacuous estimates of generalization for practical networks. 
More recent work refines norm-based PAC-Bayesian bounds by incorporating spectral norms of weight matrices, which directly control the Lipschitz constant of deep networks and lead to sharper, architecture-aware generalization guarantees \cite{neyshabur2018pac,zhou2018non,lotfi2022pac}. By combining PAC-Bayesian analysis with spectral normalization and matrix concentration inequalities, these results yield bounds depending on products of spectral norms and sums of Frobenius norms, substantially improving over purely Frobenius-based estimates for deep architectures. Such bounds are particularly well suited for convolutional networks and graph neural networks \cite{liao2020pac,ju2023generalization,brilliantov2024compositional}, where operator norms capture intrinsic stability properties. This line of work also connects PAC-Bayes to robustness against adversarial or stochastic perturbations \cite{xiao2023pac,jin2022enhancing,11027475}.

\subsection{Motivation and Contributions}
\textbf{Motivation.}
Despite the substantial progress of PAC-Bayesian generalization bounds for deep networks, existing results still exhibit several limitations that motivate the development of a unified and more interpretable framework. First, most analyses adopt isotropic Gaussian posteriors for mathematical convenience, assuming homogeneous uncertainty across all parameters. Nevertheless, such an isotropy is rarely consistent with the highly anisotropic geometry of trained deep models observed in practice. Second, many bounds depend crucially on controlling the spectral norm of weight perturbations, where the use of matrix concentration inequalities may yield loose mean or variance estimates, particularly for deep or structured architectures, thereby limiting both tightness and interpretability. Third, current approaches typically treat networks in a largely architecture-agnostic manner, overlooking the rich structural properties of modern models and the heterogeneous sensitivity of the loss landscape with respect to different layers, modules, or parameter subspaces. These limitations suggest that a unified PAC-Bayesian norm-based framework—capable of accommodating anisotropic posteriors, refined perturbation control beyond spectral norms, and explicit architectural and sensitivity-aware structures—is necessary to more faithfully capture the generalization behavior of deep learning models.

\textbf{Contributions.} Motivated by the aforementioned limitations, this work makes the following contributions toward a unified and practically relevant framework of PAC-Bayesian norm-based generalization for deep learning:

\begin{enumerate}
    \item \textbf{Unified PAC-Bayesian framework from an optimization perspective.} We establish a unified framework for PAC-Bayesian norm-based generalization bounds, in which the derivation of the generalization bound is reformulated as a stochastic optimization problem in a layer-wise manner. In contrast to conventional analyses with isotropic Gaussian posteriors fixed \textit{a priori}, our framework allows for fully anisotropic Gaussian posteriors and treats them as optimization variables. By minimizing the KL divergence term, these posterior covariances can then be identified in a closed-form with respect to a properly designed, structure-aware sensitivity matrix. This perspective explicitly couples the network sensitivity to weight perturbations and the KL complexity term through the sensitivity matrix, providing a principled mechanism to adapt the posterior geometry to the loss landscape of the trained network.
    \item \textbf{Sensitivity-aware and structure-exploiting bounds.} The sensitivity matrix characterizes the difference of network outputs to weight perturbations, in such a way that the required perturbation condition in the margin-based bound can be satisfied by explicit designs of sensitivity matrices. This also avoids the spectral norm concentration in the previous works, but instead employs the more flexible concentration inequalities for vector norms. By accommodating various structural properties, the sensitivity matrix allows for heterogeneous weight perturbations to fit diverse loss landscapes of different networks. Imposing different properties on the sensitivity matrix—such as diagonal, residual, low-rank, circulant, and Toeplitz structures—leads to a family of PAC-Bayesian generalization bounds that recover, unify, and extend several existing results as special cases. This approach enables fine-grained control of weight perturbation effects without relying solely on spectral-norm concentration.
    \item \textbf{Tighter and more interpretable generalization guarantees.} Thanks to the unified framework with flexible designs of sensitivity matrices, the resulting PAC-Bayesian norm-based generalization bounds are shown to be comparable to, and strictly tighter than, existing ones (e.g., fully-connected, residual, and convolutional networks) that rely on isotropic perturbations and/or restrictive spectral-norm concentrations. By aligning posterior anisotropy with network architecture and heterogeneous parameter sensitivity, the proposed framework yields theoretically sharper and more interpretable bounds, offering new insights into the interplay between weight norms, loss landscape, and network structure in generalization. That is, network architectures do not merely matter parameter counts --- they refine the spectral complexity with shared weights and restrict the admissible weight sensitivity to loss landscapes. As such, generalization is governed by both the compact spectral complexity with refined weight norms and the weight perturbation geometry allowed by the loss landscape, both due to network architecture.
\end{enumerate}

The paper is organized as follows. In the next section, we provide some preliminaries in statistical learning theory, especially the norm-based bound derived in the literature, together with the main proof techniques. Section III presents the proposed unified framework and summarizes the resulting generalization bounds, along with specific designs of sensitivity matrices. This is followed by Section IV, which provides detailed generalization bounds for those specific designs of sensitivity matrices. Section V concludes the paper with some detailed proofs relegated to the appendix.

\section{Preliminaries}
\label{gen_inst}

\textbf{Notation.} We interchangeably use random variables and their realizations unless they are unclear from the context. We use $a$, $\va$, $\mA$, and $\gA$ to represent the scalar, vector, matrix, and set, respectively. Accordingly, $\va_i$ is the $i$-th element of $\va$, $\mA_{i,:}$ is the $i$-th row of $\mA$, and $\mA_{ij}$ is the element at the $i$-th row and the $j$-th column in $\mA$. The operators $\Tr(\cdot)$ and $\det(\cdot)$ are the trace and determinant of a square matrix, respectively.
Given a vector $\va$, the $\ell_2$ and $\ell_{\infty}$ vector norm are given by $\Norm{\va}_2=\sqrt{\sum_i \va_i^2}$ and $\Norm{\va}_{\infty}=\max_i \Abs{\va_i}$, with $\Norm{\va}_{\infty} \le \Norm{\va}_2$. Given a matrix $\mA$, we use $\Norm{\mA}_2$ and $\Norm{\mA}_F$ to represent the spectral norm and the Frobenius norm, such that $\Norm{\mA}_2=\max_{\vx \neq \vzero} \frac{\Norm{\mA \vx}_2}{\Norm{\vx}_2}$ and $\Norm{\mA}_F^2=\sum_i \sum_j \mA_{ij}^2 = \Tr(\mA^T\mA)$.
For any vector $\vx$ and any matrices $\mA$ and $\mB$, we have $\Norm{\mA\vx}_2 \le \Norm{\mA}_2 \Norm{\vx}_2$ and $\Norm{\mA \mB}_2 \le \Norm{\mA}_2 \Norm{\mB}_2$.

Following the problem setting of \cite{neyshabur2018pac}, we consider the classification task with the input $\vx \in \mathcal{X} \triangleq \{\vx \in \mathbb{R}^n: \|\vx\|_2 \le B\}$ and the output $y = \arg \max_y f_{\vw}(\vx)[y]\in \{1,\dots,K\} $, where $f_{\vw}: \mathcal{X} \to \mathbb{R}^K$ is a function specified by a parameterized learning model (e.g., neural networks) for the classification task with the collected model parameters $\vw$. Specifically, for a $d$-layer feedforward neural network $f_{\vw}(\vx)=\mW_d \phi(\mW_{d-1}\phi(\dots \phi(\mW_1 \vx)))$ with $\phi$ being the ReLU activation function and $h_l$ the number of neurons at $l$-th layer,\footnote{For simplicity, the bias term is absorbed into the weight matrices by appending constant 1 to the node feature, which is a common trick in deep learning community.} we have the weight matrix $\mW_l \in \R^{h_{l} \times h_{l-1}}$ and $\vw=\mathrm{vec}(\{\mW_l\}_{l=1}^d)$ where $\mathrm{vec}(\cdot)$ is the vectorization of a matrix/tensor. 
Let $N=\sum_{l=1}^d h_l h_{l-1}$ denote the total number of model parameters with $h_0=n$ and $h_d=K$.
A dataset $\mathcal{S}=\{(\vx_1,y_1), \dots, (\vx_m,y_m)\}$ with $m$ training samples drawn identically and independently from an unknown distribution $\mathcal{D}$ is given to learn the model parameters.
\subsection{Norm-based Bounds}
\textbf{Margin loss.} For any distribution $\mathcal{D}$ and margin $\gamma>0$, the expected margin loss is defined as
\begin{align}
    L_{\gamma}(f_{\vw}) \triangleq \mathbb{E}_{(\vx,y) \sim \mathcal{D}} \vone\Big( f_{\vw}(\vx)[y] \le \gamma + \max_{j \neq y} f_{\vw}(\vx)[j] \Big)
\end{align}
where $\vone(\cdot)$ is the indicator function. The empirical margin loss $\hat{L}_{\gamma}(f_{\vw})$ is the estimate of $L_{\gamma}(f_{\vw})$ with the average over the training dataset $\mathcal{S}$. When $\gamma=0$, $L_{0}(f_{\vw})$ and $\hat{L}_{0}(f_{\vw})$ denote the expected risk and the training error, respectively.
Intuitively, the margin loss captures how confidently (measured by the margin) the model predicts the correct label relative to the incorrect ones. It depends explicitly or implicitly on weight norms/directions, the network's depth, activation functions, loss landscape, and training regularization.

\textbf{PAC-Bayesian generalization bounds.}  Given the stochastic classifier $f_{\vw}$ with the prior distribution $P$ and the corresponding posterior distribution $Q$ in the form of $\vw+\vu$, where $\vu$ is a random weight perturbation,
with probability at least $1-\delta$, we have the PAC-Bayesian generalization error bound \cite{mcallester2003simplified}
\begin{align}
    \MoveEqLeft \E_{\vu} [L_0(f_{\vw+\vu})] \le \E_{\vu} [\hat{L}_0(f_{\vw+\vu})] \nn \\
    &+ 2 \sqrt{\frac{2 \KL(\vw+\vu || P) + \ln \frac{2m}{\delta}}{m-1}},
\end{align}
where $\KL(\cdot || \cdot)$ is the KL divergence of two distributions.
For any $\gamma,\delta > 0$, for any $\vw$ and random perturbation $\vu$ subject to the perturbation condition
\begin{align} \label{eq:pert-condition}
    \mathbb{P}_{\vu} [\max_{\vx \in \gX} \|f_{\vw+\vu}(\vx)-f_{\vw}(\vx)\|_{\infty} < \frac{\gamma}{4} ] \ge \frac{1}{2}
\end{align}
with probability at least $1-\delta$, we have the norm-based PAC-Bayesian generalization error bound \cite{mcallester2003simplified,langford2002pac}
\begin{align}
    L_0(f_{\vw}) \le \hat{L}_{\gamma}(f_{\vw}) + 4 \sqrt{\frac{\KL(\vw+\vu || P) + \ln \frac{6m}{\delta}}{m-1}}.
\end{align}

Let the random weight perturbation be $\vu=(\vu_1^T,\dots,\vu_d^T)^T$ where $\vu_l=\mathrm{vec}(\mU_l)$ is the weight perturbation at $l$-th layer with $\mU_l$ being the perturbation added to $\mW_l$.
Let $h = \max_l h_l$ for simplicity.
Assuming $\vu \sim \mathcal{N}(\vzero,\sigma^2 \mI)$, for any $B, d, h > 0$, for a $d$-layer neural network with ReLU activation functions, for any $\gamma, \delta > 0$, with probability at least $1-\delta$, we have the following spectrally-normalized generalization bound \cite{neyshabur2018pac}
\begin{align} \label{eq:gen-bound-origin}
    L_0(f_{\vw}) \le \hat{L}_{\gamma}(f_{\vw}) + \mathcal{O} \left( \sqrt{ \frac{B^2 d^2 h \ln(dh) \Phi(\vw) + \ln \frac{dm}{\delta}}{\gamma^2 m}} \right)
\end{align}
with spectral complexity $\Phi(\vw)=\prod_{l=1}^d \|\mW_l\|_2^2 \sum_{l=1}^d \frac{\|\mW_l\|_F^2}{\|\mW_l\|_2^2}$. The term $\prod_{l=1}^d \|\mW_l\|_2^2$ indicates the global Lipschitz constant of the neural networks, which measures how sensitive the output is to the inputs. The term $\sum_{l=1}^d \frac{\|\mW_l\|_F^2}{\|\mW_l\|_2^2}$ reflects how concentrated or flat each weight matrix is, and is also related to the model capacity measured with Rademacher complexity.

The above norm-based bound bridges the function space and parameter space through $\gamma$ and $\Phi(\vw)$. The margin $\gamma$ is defined in the output function space, while spectral complexity $\Phi(\vw)$ with respect to weight norms measures the parameter space. The generalization error gap is bounded by the ratio $\frac{\Phi(\vw)}{\gamma^2}$, favoring low spectral complexity and high margin for a tight bound. The spectral complexity usually measures the ``intrinsic capacity'' or ``sensitivity'' of the network architectures with model weights, which can be reduced by explicit regularization, such as weight decay, spectral norm regularization, or implicit regularization like SGD. On the contrary, the margin measures the ``robustness'' or ``confidence'' of the model's predictions on the data, which can usually be ensured by, e.g., max-margin training and label smoothing. The training process is usually trying to find a good tradeoff between them, aligning perfectly with the norm-based generalization bound.

\subsection{Proof Techniques}
\textbf{Spectral normalization.}
In the above spectrally normalized generalization error bound, the spectral norms of $\{\mW_l\}_{l=1}^d$ are normalized across the layers. Specifically, let $\beta=\sqrt[d]{\prod_{l=1}^d \Norm{\mW_l}_2}$ and the normalized weights, i.e., $\Tilde{\mW}_l = \frac{\beta}{\Norm{\mW_l}_2}\mW_l$ are considered. Such a weight normalization does not change the model output, nor the empirical and expected (margin) loss, due to the homogeneity of the ReLU \cite{neyshabur2018pac}. Given that the weight normalization does not change the spectral complexity, i.e., $\Phi(\vw)=\Phi(\Tilde{\vw})$, it is assumed in \cite{neyshabur2018pac} to consider the same spectral norm across layers, i.e., $\Norm{\mW_l}_2=\beta$ for all $l$, without loss of generality. 

The key gradients of the spectrally-normalized bounds consist of the following steps:
\begin{enumerate}
    \item Given any fixed $\hat{\beta}$, for all $\beta$ such that $\Abs{\hat{\beta}-\beta} \le \frac{1}{d}\beta$, choose a proper $\sigma^2$ as a function of $\hat{\beta}$ to satisfy the perturbation condition based on the concentration inequality of spectral norm $\mU_l$ and the perturbation bound \eqref{eq:pert-bound-origin};
    \item Given the choice of $\sigma^2$, compute the KL divergence and bound the difference between expected and empirical margin loss for the fixed $\hat{\beta}$;
    \item Given the bound for a fixed $\hat{\beta}$, derive the final generalization error bound for all possible $\hat{\beta}$ by taking a union bound over all chosen $\hat{\beta}$ so that every $\beta$ can be covered.
\end{enumerate}

The norm-based generalization bound in \cite[Lemma 1]{neyshabur2018pac} is under the perturbation condition in \eqref{eq:pert-condition}, which is enforced by a deterministic perturbation bound and a concentration inequality on the spectral norm.
Specifically, given the fact that $\Norm{\cdot}_\infty \le \Norm{\cdot}_2$, the perturbation condition with probability at least $\frac{1}{2}$ is satisfied by the perturbation bound with respect to the spectral norm $\Norm{\mU_l}_2$ and its concentration inequality.

\textbf{Perturbation bound.}
The generalization bound \eqref{eq:gen-bound-origin} is built upon a key lemma of the perturbation bound in \cite[Lemma 2]{neyshabur2018pac}, where the change of the output due to weight perturbation is upper bounded as
\begin{align} \label{eq:pert-bound-origin}
    \Norm{f_{\vw+\vu}(\vx)-f_{\vw}(\vx)}_2 \le e B \prod_{l=1}^d \|\mW_l\|_2 \sum_{l=1}^d \frac{\|\mU_l\|_2}{\|\mW_l\|_2}
\end{align}
when $\|\mU_l\|_2 \le \frac{1}{d}\|\mW_l\|_2$, for any weight perturbation $\vu$ and any bounded input $\vx$ subject to $\Norm{\vx}\le B$. 
Notice that $\prod_{l=1}^d \|\mW_l\|_2$ is the multiplicative chain of Lipschitz constants across layers, and $\sum_{l=1}^d \frac{\|\mU_l\|_2}{\|\mW_l\|_2}$ linearly aggregates perturbation sensitivity via a weighted sum of normalized perturbations.
The above perturbation bound is essential to derive the tight generalization bound, and has been extended to adversarial settings \cite{xiao2023pac,jin2022enhancing,11027475} and graph neural networks \cite{liao2020pac,ju2023generalization,brilliantov2024compositional}.

\textbf{Concentration inequality.}
Given an isotropic Gaussian distribution $\vu \sim \mathcal{N}(0,\sigma^2 I)$, the concentration inequality with respect to $\|\mU_l\|_2$ is given by
\begin{align}
    \mathbb{P}_{\vu_i}(\Norm{\mU_l}_2 > t) \le 2 h \exp\left({-\frac{t^2}{2h\sigma^2}}\right)
\end{align}
for which an upper bound of $\|\mU_l\|_2$ can be derived with probability at least $\frac{1}{2}$ by setting $t$ as a function of $h$ and $\sigma^2$. This can be used to upper-bound the perturbation bound in order to satisfy the perturbation condition in \eqref{eq:pert-condition}. The largest permissible choice of $\sigma^2$ as a function of $h,d,B,\gamma$ and $\prod_{l=1}^d \Norm{\mW_l}_2$ is to ensure the network output shift under $\vu$ stays within the margin $\gamma$ with high probability, while keeping the KL-term as small as possible.

\textbf{Limitations.} Albeit a promising approach, there are two limitations in existing works. First, most existing works consider the isotropically Gaussian distributed weight perturbation, i.e., $\vu \sim \mathcal{N}(\vzero,\sigma^2 \mI)$, assuming that the weight perturbations in different directions are identically uncertain and independently play the same role. In reality, some weight patterns (e.g., convolutional filters or graph diffusion matrices) or weight directions (in the parameter space) are more sensitive to the perturbations, and therefore hurt the network's margin much more than others. The isotropic perturbation treats every weight as equally sensitive, and thus cannot capture the ``sharpness'' of weight directions or patterns. 
Second, although the perturbation bound \eqref{eq:pert-bound-origin} appears elegant and tight, it renders the choices of $\sigma^2$ reliance on the concentration inequality of matrix spectral norm, resulting in high inflexibility to take into account any potential structure of the model weights. This motivates us to reconsider the necessity of the perturbation bound for the sake of flexible and unifying designs.

\section{A Unified Framework}
\label{headings}
To address the above issues, we propose a unified framework by introducing two auxiliary variables.
\begin{enumerate}
    \item Consider an anisotropic weight perturbation, e.g., $\vu \sim \mathcal{N}(0,\sigma^2 \mR)$ to capture different sensitivities to weight perturbations, where $\mR$ will be optimized to minimize the KL divergence and, in turn, the complexity measure;
    \item Introduce a sensitivity matrix $\mA$ to encode different network structures, so as to impose different sensitivities on the perturbation to the output, and enable the flexible designs of perturbation bounds.  
\end{enumerate}

Therefore, the derivation of the generalization bound can be done by solving the following optimization problem
\begin{subequations} \label{eq:general-opt}
\begin{align} 
    \min_{\sigma^2, \mR, \mA} \quad & \KL(\vw+\vu || P)\\
    \st \quad &\mathbb{P}_{\vu \sim \gN(0, \sigma^2 \mR)} (\Norm{\mA \vu}_2 < \frac{\gamma}{4}) \ge \frac{1}{2},\\
    & \Norm{f_{\vw+\vu}(\vx)-f_{\vw}(\vx)}_{\infty} \le \Norm{\mA \vu}_2,
\end{align}
\end{subequations}
where the $\ell_\infty$ norm is upper-bounded by the $\ell_2$ norm in such a way that proper concentration inequalities of the $\ell_2$ norm ensure a closed-form solution $\mR$ with respect to $\mA$. Here $\mA$ is referred to as a Jacobian-like sensitivity matrix, which is a function of the weights and the inputs. The choices of $\mA$ play a key role in tightening the PAC-Bayes generalization bounds through $\mR$. When it makes $\mR$ aligning with flat directions of the loss landscape, the expected margin loss drift can be minimized, resulting in a lower empirical loss under weight perturbation. When suppressing the variance of $\mR$ in sensitive directions, it avoids large KL divergence, giving us a lower model complexity penalty.

However, directly optimizing \eqref{eq:general-opt} is intractable, as it involves the perturbation condition in a probabilistic form.
As such, it enables us to employ the concentration inequality of the $\ell_2$-norm of the vector $\mA \vu$, rather than the spectral norm of $\mU_l$, to derive more flexible perturbation bounds, ultimately yielding more adaptable bounds.

\subsection{The General Framework}
\subsubsection{Anisotropic weight perturbation} To capture the diverse sensitivities of the model's outputs to weight perturbations, we relax the commonly used assumption of the isotropic Gaussian distribution to the anisotropic settings such that $\vu \sim \mathcal{N}(\vzero,\sigma^2 \mR)$ where $\sigma^2$ is a tunable variable shared by both prior and posterior, and $\mR \triangleq \mathbb{E}[\vu \vu^T]$ is the positive definite covariance matrix of weight perturbations across the $d$-th layer. 
This allows us to encode in $\mR$ different sensitivities and any correlations between weight perturbations.
With a general $\mR \ne \mI$, we can down-weight the ``sharp'' directions where a tiny perturbation could collapse the margin, and up-weight the ``flat'' directions when there is some slack, in such a way that the resulting PAC-Bayesian norm-based bound could reflect the actual geometry of the loss landscape. 

Moreover, by optimizing $\mR$ towards the KL minimization, we unlock the potential of achieving much tighter norm-based bounds.
A larger variance $\mR$ implies more stochasticity, helping reduce margin loss under perturbation (cf. a flatter loss landscape) at the cost of increased KL. On the contrary, reducing the variance can decrease the KL divergence, but it also increases the sensitivity.
As such, the PAC-Bayesian bound encourages a structured design of $\mR$ to put variance in directions where the loss is stable, and suppress variance where the margin is sensitive.

There is also a tradeoff between $\sigma^2$ and $\mR$ due to the perturbation condition. Too large $\sigma^2$ leads to excessive perturbations that make the condition fail, while a poorly structured $\mR$, e.g., large variances in sensitive directions, may also break the condition. In reality, $\sigma^2$ should be as large as possible to minimize the KL divergence while satisfying the condition, and $\mR$ should align with the model sensitivity to avoid perturbing unstable directions.

\subsubsection{Alternative perturbation condition}
With the consideration of anisotropic Gaussian distributed weight perturbations, we need an upgraded concentration inequality for the spectral norm of correlated weight perturbation $\mU_l$ as in \cite{11027475}. Nevertheless, we can circumvent such a requirement by constraining the perturbation bound with the following alternative, i.e.,
\begin{align} \label{eq:pert-cond}
    \|f_{\vw+\vu}(\vx)-f_{\vw}(\vx)\|_{\infty} \le \|\mA \vu\|_2
\end{align}
for any $\vw$ and $\vx \in \gX$, 
where $\mA \in \R^{N \times N}$ is a sensitivity matrix, serving as a proxy for a global Jacobian. As such, the original perturbation condition is enforced by bounding the change in the network’s output using a Jacobian-type matrix $\mA$.
By such an updated perturbation condition, the original weight-space perturbation can be translated to function-space stability.
Thus, to satisfy the perturbation condition, we need to guarantee that
\begin{align}
    \mathbb{P}_{\vu \sim \gN(\vzero,\sigma^2 \mR)} \big[ \Norm{\mA \vu}_2 < \frac{\gamma}{4} \big] \ge \frac{1}{2}.
\end{align}

\subsubsection{Concentration inequality of the $\ell_2$ vector norm}
Let $\vu=\sigma \mR^{\frac{1}{2}}\vz$ with $\vz \sim \gN(\vzero,\mI)$ and $\mM=\mA \mR \mA^T$ positive semi-definite, we have
\begin{align}
    \Norm{\mA \vu}_2^2 = \sigma^2 \vz^T \mM \vz, \qquad \E [\Norm{\mA \vu}_2^2] = \sigma^2 \Tr(\mM).
\end{align}
By \cite[Proposition 1.1]{hsu2012tail}, we have
\begin{align}
  \MoveEqLeft  \mathbb{P}\Big(\Norm{\mA \vu}_2^2 \le \sigma^2(\Tr(\mM) + \sqrt{4t} \Norm{\mM}_F + 2t \Norm{\mM}_2)\Big) \nn \\
  &\ge 1 - \exp(-t)
\end{align}
for all $t>0$. 
This can be similarly done by the Hanson-Wright inequality \cite{rudelson2013hanson} that bounds the quadratic form of sub-Gaussian random vectors, such as $\Norm{\mA \vu}_2^2$. It also concentrates the quadratic form on its mean $\Tr(\mM)$ with a deviation term controlled via $\Norm{\mM}_F$ and $\Norm{\mM}_2$.

Let $1 - \exp(-t) = \frac{1}{2}$, we have $t=\ln 2$. 
Thus, we conclude that, with probability at least $\frac{1}{2}$, it holds
\begin{align}
     \Norm{\mA \vu}_2^2 &\le \sigma^2 (\underbrace{\Tr(\mM) + \sqrt{4\ln2} \Norm{\mM}_F + 2\ln 2 \Norm{\mM}_2}_{\triangleq \; \Gamma(\mM)})
\end{align}

It is worth noting here that the above concentration inequality works directly on the $\ell_2$ vector norm of $\mA \vu$, rather than the spectral norm of $\mU_l$ as those in \cite{neyshabur2018pac,xiao2023pac,jin2022enhancing,11027475,liao2020pac}. This brings in two benefits: 1) it is more flexible to consider anisotropic Gaussian distributions for the vector norm than the spectral norm as in \cite{11027475}; and 2) this avoids layer-wise consideration of spectral norms and the need of a union bound over layers of concentration inequalities, therefore eliminating a $\ln (dh)$ factor in the choices of $\sigma^2$ as in \cite{neyshabur2018pac,xiao2023pac,jin2022enhancing,11027475,liao2020pac} and many others.

By letting $\sigma^2 \Gamma(\mM) = \frac{\gamma^2}{16}$, we have
\begin{align} \label{eq:perb-cond-Gamma}
    \frac{1}{\sigma^{2}} = \frac{16 \Gamma(\mM)}{\gamma^2}
\end{align}
to satisfy the perturbation condition. Note here that the choice of $\sigma^2$ scales as $\gamma^2$, and the perturbation condition still holds if some upper bounds of $\Gamma(\mM)$ are in place in \eqref{eq:perb-cond-Gamma}.

Note also here that the choice of $\sigma^2$ is also used for the isotropic Gaussian prior, which usually cannot depend on the learned weights $\{{\mW}_l\}_{l=1}^d$ that may be a part of $\mM$ through the flexible design of $\mA$. To address this issue, we follow the strategy in \cite{neyshabur2018pac} that has been widely accepted in the follow-up works in e.g., \cite{xiao2023pac,jin2022enhancing,11027475,liao2020pac}, to add some adjustment on the choice of $\sigma^2$ according to the approximated weights $\{\hat{\mW}_l\}_{l=1}^d$.

\subsubsection{KL divergence minimization} Given the Gaussian prior $P=\gN(\vzero,\sigma^2\mI)$ and posterior $Q=\gN(\vw, \sigma^2 \mR)$ with the same $\sigma$, the KL divergence can be written by \cite{pardo2018statistical}
\begin{align}
    \KL(Q||P) &= \frac{1}{2} \left( \frac{\Norm{\vw}_2^2}{\sigma^2} + \Tr(\mR) - \log \det \mR - \dim (\mR) \right) \nn \\
    &= \frac{1}{2} \left( \frac{16\Norm{\vw}_2^2}{\gamma^2} \Gamma(\mA \mR \mA^T) \right. \nn \\
    & \qquad  \left. + \Tr(\mR) - \log \det \mR - \sum_{l=1}^d h_l h_{l-1} \right)
\end{align}
due to the choice of $\sigma^2$.
Note that $\frac{\Norm{\vw}_2^2}{\sigma^2}$ penalizes the large weights related to posterior uncertainty, $\Tr(\mR)$ is the total variance budget across weight directions, $\log \det \mR$ is the volume of correlation for anisotropic uncertainty, and $\dim(\mR)$ is the dimension of the Gaussian distributed vector $\vu$.
To tighten the generalization error bound, we can minimize $\KL(Q||P)$ by choosing the proper covariance matrix $\mR$. As such, we formulate an optimization problem
\begin{align}
    \min_{\mR \in \gS_{+}} \gJ(\mR) \triangleq \frac{16\Norm{\vw}_2^2}{\gamma^2} \Gamma(\mA \mR \mA^T) + \Tr(\mR) - \log \det \mR
\end{align}
where $\gS_{+}$ is the space of symmetric positive semi-definite matrices. By $\min_{\mR \in \gS_{+}} \gJ(\mR)$, we end up with $\KL(Q||P)$ as a function of $\mA$ that is carefully designed to satisfy the perturbation condition in \eqref{eq:pert-cond}, followed by the final generalization error bound with each specific design of $\mA$. Although $\gJ(\mR)$ is convex on $\mR$, it is still challenging to find a neat closed-form solution using, e.g., KKT conditions, due to the non-differentiability of the spectral norm. Therefore, for simplicity and traceability, we instead minimize a tractable upper bound of $\gJ(\mR)$ to have a closed-form solution to $\mR$.

For the positive demi-definite matrix $\mM=\mA \mR \mA^T$, we have
\begin{align}
    \Norm{\mM}_2 \le \Norm{\mM}_F \le \Tr(\mM)
\end{align}
where the equalities hold when $\mM$ has rank-1.
Therefore, we set $\frac{1}{\sigma^2}=\frac{16 \kappa}{\gamma^2} \Tr(\mM)$ to satisfy the perturbation bound, where 
\begin{align}
    \kappa=1+2\ln2 + \sqrt{4\ln2}.
\end{align}
As such, we end up with an upper bound of $\gJ(\mR)$ as
\begin{align*} 
\gD(\mR) \triangleq \frac{16 \kappa\Norm{\vw}_2^2}{\gamma^2} \Tr(\mA \mR \mA^T) + \Tr(\mR) - \log \det \mR.
\end{align*}
 Since $\gD(\mR)$ is convex over $\mR$, the minimization of $\gD(\mR)$ using KKT conditions yields a closed-form solution, denoted by $\mR^*=\arg \min_{\mR \in \gS_{+}} \gD(\mR)$, i.e.,
\begin{align}
    \mR^* = \left(\mI + \frac{16 \kappa\Norm{\vw}_2^2}{\gamma^2} \mA^T \mA  \right)^{-1},
\end{align}
which shrinks variance along sensitive directions with large $\mA^T \mA$. 
If $\mA^T \mA$ is diagonal, then $\mR^*$ is diagonal, recovering the i.i.d. case; If $\mA^T \mA$ is low-rank or with spectral decay, then $\mR^*$ effectively allocates less variance to sensitive directions, tightening the KL bound.  
Note also that $\mR^*$ is also dependent on the margin $\gamma$. When $\gamma$ is small, $\mR^*$ relies on the choice of $\mA$ so that the directional sensitivity dominates; when $\gamma$ is large, $\mR^*$ reduces to the isotropic case so that the sharpness of the loss landscape does matter.

Plugging $\mR^*$ in the KL divergence, we have
\begin{align}
    \KL(Q||P) = \frac{1}{2} \left( \frac{\Norm{\vw}_2^2}{\sigma^2} + \Tr(\mR^*) - \log \det \mR^* - N \right). 
\end{align}
Together with the proper choice of $\sigma^2$ and the design of $\mA$ to satisfy the perturbation condition, we end up with a new form of PAC-Bayesian generalization error bounds.

\begin{figure*}
\begin{subequations} \label{eq:main-opt}
\begin{align} 
    \min_{\sigma^2, \mR_l, \mA_l} \quad & \KL(\sigma^2,\mR_l) \triangleq \frac{1}{2} \sum_{l=1}^d\frac{\|\mW_l\|_F^2}{\sigma^2} + \Tr(\mR_l) - \log \det \mR_l - \dim(\mR_l)\label{eq:main-opt-obj}\\
    \mathrm{s.t.} \quad & \mathbb{P}_{\vu_l \sim \gN(\vzero,\sigma^2 \mR_l)} \Big[\sum_{l=1}^d \Norm{\mA_l \vu_l}_2^2 < \frac{\gamma^2}{16} \Big] \ge \frac{1}{2},\label{eq:perb-cond}\\
    &\|f_{\vw+\vu}(\vx)-f_{\vw}(\vx)\|_{\infty}^2 \le \sum_{l=1}^d\|\mA_l \vu_l\|_2^2, \label{eq:perb-bound}
\end{align}
\end{subequations}
\hrule
\end{figure*}

\subsection{Block-diagonal Simplification}
For simplicity, we assume there are no correlations of weight perturbations across layers, where the covariance matrix $\mR$ of weight perturbation possesses a block-diagonal structure, i.e., $\mR = \mathrm{blkdiag}(\mR_1, \dots, \mR_d)$ with each diagonal block $\mR_l$ up to optimization.
Similarly, we let $\mA = \mathrm{blkdiag}(\mA_1, \dots, \mA_d)$ with $\mA_l$ being carefully designed for each layer $l$, as a global measure of sensitivity without respecting inter-layer structures.

With the block-diagonal simplification, we advocate a unified framework to derive the norm-based PAC-Bayesian generalization error bounds by solving the optimization problem in \eqref{eq:main-opt} on the top of the next page. Note that the objective \eqref{eq:main-opt-obj} is to find the minimized KL divergence by choosing the proper $\mR_l$ and $\sigma^2$ to tighten the upper bound, the first constraint \eqref{eq:perb-cond} comes from the perturbation condition with properly chosen $\sigma^2$, and the second constraint \eqref{eq:perb-bound} is to guide the choices of $\mA_l$.

In general, it is challenging to solve the problem directly. Instead, we propose the following route to find some feasible solutions, which end up with new generalization error bounds for different designs of $\{\mA_l\}_{l=1}^d$. The key is to make a good tradeoff between $\mA_l$ and $\mR_l$, where the former captures weight perturbation as much as possible to satisfy the perturbation condition, and the latter aligns sensitivity to reduce the influence of weight perturbation as much as possible for tighter bounds.

\subsubsection{Figure out the prior variance $\sigma^2$ to satisfy the perturbation condition} In general, we need to choose a proper $\sigma^2$ as a function of $\gamma$, $\mA_l$ and $\mR_l$ to satisfy the perturbation condition \eqref{eq:perb-cond} with the following chain of inequalities
\begin{align}
\Norm{\mA \vu}_2^2 =  \sum_{l=1}^d \Norm{\mA_l \vu_l}_2^2 &\le \sigma^2  \Gamma(\mM) \nn \\ 
&\le \sigma^2  \sum_{l=1}^d \Gamma(\mM_l) \le \frac{\gamma^2}{16} 
\end{align}
where $\mM_l=\mA_l \mR_l \mA_l^T$ and the second inequality is due to 
\begin{align}
&\Norm{\mM}_F=\sqrt{\sum_l \Norm{\mM_l}_F^2} \le \sum_l \Norm{\mM_l}_F,\\
&\Norm{\mM}_2 = \max_l \Norm{\mM_l}_2 \le \sum_l \Norm{\mM_l}_2,
\end{align}
for block diagonal matrix $\mM=\mathrm{blkdiag}(\mM_1,\dots,\mM_d)$.
As $\mM_l$ is positive semi-definite, we have $\Norm{\mM_l}_2 \le \Norm{\mM_l}_F \le \Tr(\mM_l)$, and therefore the perturbation condition can be rewritten as
\begin{align}
   \sum_{l=1}^d \Norm{\mA_l \vu_l}_2^2 &\le \sigma^2  \sum_{l=1}^d \Gamma(\mM_l) \nn \\
   &\le \sigma^2 \kappa \sum_{l=1}^d \Tr(\mM_l) \le \frac{\gamma^2}{16}. \label{eq:chain-of-ineq}
\end{align}
By setting 
\begin{align} \label{eq:sigma-diag}
\frac{1}{\sigma^2} = 
\frac{16 \kappa}{\gamma^2} \sum_{l=1}^d \Tr({{\mM}}_l), 
\end{align}
the original perturbation condition can be satisfied with probability at least $\frac{1}{2}$. 
As said earlier, $\sigma^2$ should be independent of the learned weight $\{\mW_l\}_{l=1}^d$, thus 
$\sigma^2$ should be adjusted later, when specifying $\mA_l$ as a function of $\{\hat{\mW}_l\}_{l=1}^d$, yet to obtain the optimized $\mR_l$ we still proceed with the $\sigma^2$ in \eqref{eq:sigma-diag} for convenience. 

\subsubsection{Minimize KL} Given the choice of $\sigma^2$, we have the upper bound $\gJ(\mR) \le \sum_l \gD(\mR_l)$ where
\begin{align} \label{eq:D_Rl_diag}
    \gD(\mR_l)=\frac{16 \kappa\Norm{\vw}_2^2}{\gamma^2} \Tr(\mA_l \mR_l \mA_l^T) + \Tr(\mR_l) - \log \det \mR_l.
\end{align}
To figure out the optimal $\mR_l^*$ as a function of $\gamma$ and $\mA_l$, by KKT conditions, we have
\begin{align} \label{eq:opt-Rl}
    \mR_l^* = \left(\mI + \frac{16 \kappa\Norm{\vw}_2^2}{\gamma^2} \mA_l^T \mA_l  \right)^{-1}.
\end{align}
Note that the design of $\mR_l^*$ shrinks variance along sensitive directions (those with large $\mA_l^T\mA_l$), which aligns with the idea of robust weight perturbation. When $\mA_l^T\mA_l$ is diagonal with independent perturbation across dimensions, $\mR_l^*$ is diagonal, recovering somehow the i.i.d. setting. When $\mA_l^T\mA_l$ is of low-rank or has a spectrum with decay, $\mR_l^*$ allocates less variance to sensitive directions, which may tighten the KL divergence.

\subsubsection{Design sensitivity matrices} The goal of designing $\mA_l$ is to capture the local Lipschitz sensitivity of the network output to the weight perturbation $\vu_l$. It turns out to be finding a surrogate Jacobian-like matrix that captures the layer-wise or neuron-wise impact of perturbations.
with certain imposed structures to satisfy the perturbation bound \eqref{eq:perb-bound}, such as diagonal, low-rank, and Toeplitz structures. Some possible designs are as follows:
\begin{align}
\begin{array}{ll}
     \text{Diagonal:} &  \mA_l =  e\sqrt{d}B \prod_{i\ne l} \Norm{\mW_i}_2  \mI \\
     \text{Residual:} &  \mA_l =  e\sqrt{d}B \prod_{i\ne l} (\Norm{\mW_i}_2+1)  \mI \\
     \text{Low-rank:} & \mA_l = \sqrt{d} \mV_l \mathrm{diag}(\sigma_{l,1},\dots,\sigma_{l,K},0,\dots,0)\mV_l^T, \\
     &\text{with } \sigma_{l,k} = B\prod_{i \ne l} \Norm{\mW_i}_2\\
     \text{Circulant:} & \mA_l = \sqrt{d} B \prod_{i\ne l} \Norm{\mV^H \vw_i}_{\infty} \mV_{[1:K]} \mV_{[1:K]}^H\\
     &\text{with $\vw_i \in \R^h$ is the vectorized convolutional}\\ &\text{filter of $i$-th layer with circulant padding.} \\
     \text{Toeplitz:} & {\mA}_l =  \frac{ e\sqrt{d}B  \prod_{i\ne l} \Norm{\mW_i}_2}{\min_{\omega}\psi(\omega))}\mT \mP\\
     &\text{with $\mT$ a Toeplitz matrix with symbol $\psi(\omega)$}\\ &\text{and $\mP \in \R^{h^2 \times k}$ for kernel-Toeplitz mapping.} 
\end{array}      
\end{align}
where $\mV_l$ is the collection of the right singular vectors of the Jacobian matrix of $l$-th layer, $\mV_{[1:K]}$ is the collection of $K$ columns of the discrete Fourier transform (DFT) matrix, and the additional factor $e$ is due to the consideration of the weights in choosing $\sigma^2$ in the prior.

For the diagonal $\mA_l$, we aim to explore the worst-case sensitivity of all directions in each layer with an equal strength, yet with maybe distinct strengths across layers. By setting $\mA_l$ to be a low-rank matrix, we wish to explore some typical directions extensively while ignoring other directions. In contrast, with $\mA_l$ being circulant matrices, we want to mimic the behavior of convolutional layers, and explore the impact of parameter sharing in generalization. Finally, by considering $\mA$ being graph Laplacian matrices, we investigate the influence of graph structures in generalization. In fact, $\mA_l$ could be also data-dependent as a function of $\vx$, yet for simplicity, we employ $\Norm{\vx}_2 \le B$ instead for any input $\vx$.

Different choices of sensitivity structures encode distinct inductive biases. The diagonal sensitivity corresponds to the worst-case and parameter-wise robustness, whilst the low-rank sensitivity reflects the output-driven or task-aligned complexity. In contrast, the circulant structure models global and periodic correlations across parameters, while the Toeplitz structure captures smoothness and locality induced by weight sharing.

\subsubsection{Obtain final bounds} 
With the adjusted ${\sigma}^2$
and the optimized $\mR_l^*$ in \eqref{eq:opt-Rl}, we can obtain the resulting KL divergence in \eqref{eq:main-opt-obj} or its upper bounds, i.e.,
\begin{align}
     \KL(\{\mA_l\}_{l=1}^d) &=  \frac{1}{2} \sum_{l=1}^d \frac{\|\mW_l\|_F^2}{{\sigma}^2} + \Tr(\mR_l^*) \nn \\
    &\qquad \qquad - \log \det \mR_l^* - h_lh_{l-1} 
\end{align}
by which we end up with new generalization error bounds
\begin{align} \label{eq:gene-bounds-all}
    L_0(f_{\vw}) \le \hat{L}_{\gamma}(f_{\vw}) + \gO \left( \sqrt{ \frac{ B^2 \Delta(d,h,\vw) + \ln \frac{dm}{\delta}}{\gamma^2m}} \right)
\end{align}
where 
\begin{align}
\Delta(d,h,\vw) = \left\{
    \begin{matrix}
        d^2h^2\Phi(\vw), & \text{$\mA_l$ diagonal}\\
        d^2h^2\Phi^{\mathrm{rn}}(\vw), & \text{$\mA_l$ residual}\\
        d^2K\Phi(\vw), & \text{$\mA_l$ low-rank}\\
        d^2K{\Phi}^{\mathrm{circ}}(\vw), & \text{$\mA_l$ circulant}\\
        d^2k\Phi^{\mathrm{toep}}(\vw), & \text{$\mA_l$ Toeplitz}
    \end{matrix}
    \right.
\end{align}
with 
\begin{align}
    \Phi(\vw)&=\prod_{l=1}^d \Norm{\mW_l}_2^2  \sum_{l=1}^d \frac{\Norm{\mW_l}_F^2}{\Norm{\mW_l}_2^2}, \\ 
    \Phi^{\mathrm{rn}}(\vw)&=\prod_{l=1}^d (\Norm{\mW_l}_2+1)^2  \sum_{l=1}^d \frac{\Norm{\mW_l}_F^2}{(\Norm{\mW_l}_2+1)^2},\\
    {\Phi}^{\mathrm{circ}}(\vw)&=\prod_{i\ne l} \Norm{\mV^H \vw_i}_{\infty}^2 \sum_{l=1}^d \Norm{\vw_l}_2^2,\\
    \Phi^{\mathrm{toep}}(\vw)&=\frac{\psi_{\max}^2}{\psi_{\min}^2} \prod_{i\ne l} \Norm{\vw_i}_1^2 \sum_{l=1}^d \Norm{\vw_l}_2^2
\end{align}
with the variables specified later. 
It is worth noting that, each choice of sensitivity matrix $\mA_l$ defines a PAC-Bayes bound that is valid for networks whose perturbation geometry is compatible with $\mA_l$. With certain structured choices (e.g., residual, low-rank, circulant, and Toeplitz) of $\mA_l$, we end up with strictly sharper bounds when the network architecture enforces the corresponding weight sharing or spectral constraints. When the architecture does not enforce such a structure, the bound remains valid only after relaxing $\mA_l$ to a worst-case (diagonal) form, leading to looser generalization guarantees. 

Finally, we form a recipe of deriving spectrally-normalized PAC-Bayesian generalization bounds, as in Fig.~\ref{fig:recipe}.

\begin{figure}[t]
\begin{center}
\includegraphics[width=0.3\textwidth]{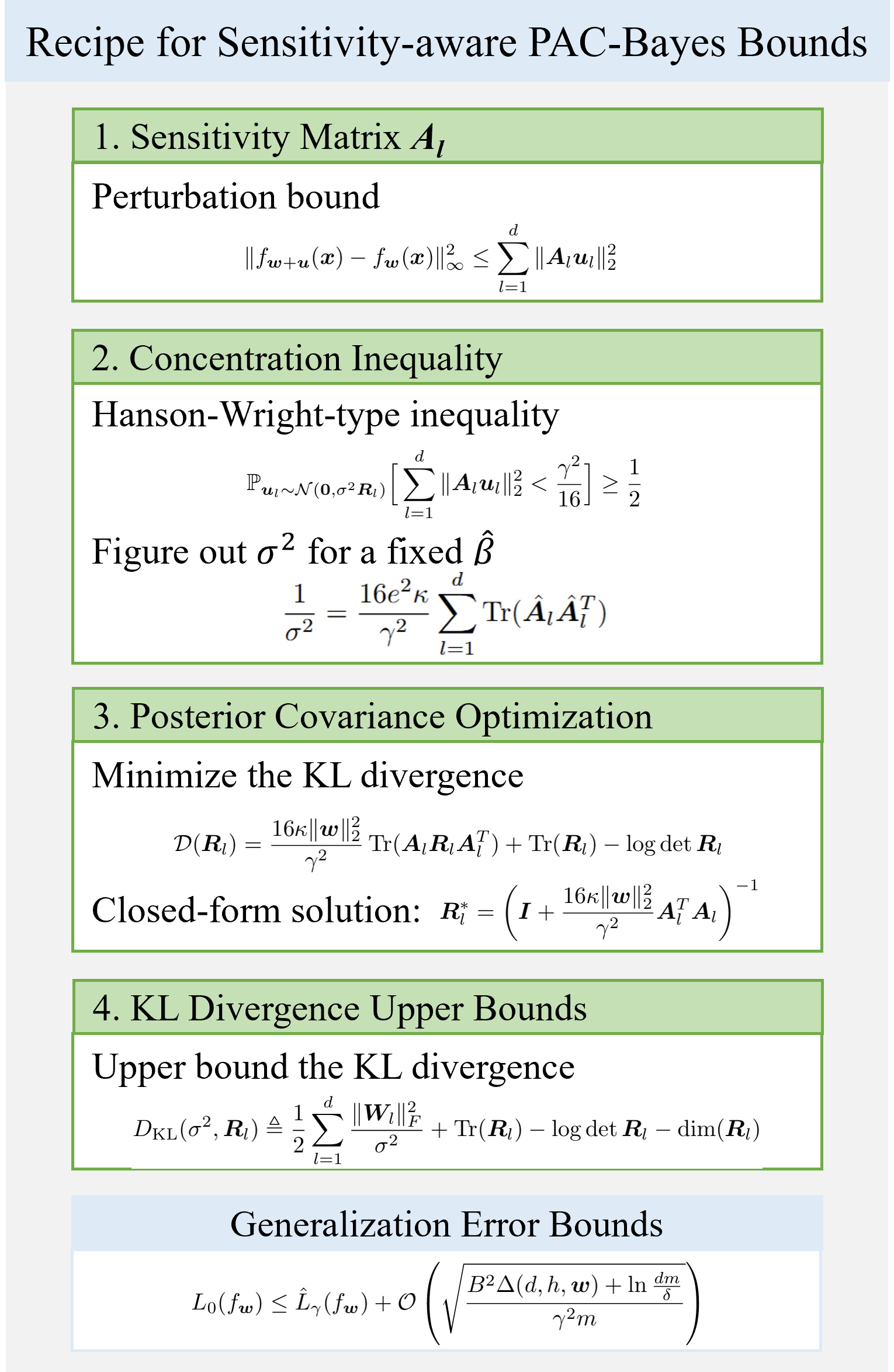}
\end{center}
\caption{The recipe of spectrally-normalized sensitivity-aware PAC-Bayesian generalization bounds.}
\label{fig:recipe}
\end{figure}

\section{Specific Sensitivity Matrices}
\label{others}
Given the above recipe, we only need to find a set of proper sensitivity matrices $\{\mA_l\}_{l=1}^d$ so that the perturbation condition is satisfied. After doing so, we can proceed directly with the computation of the optimized $\mR_l^*$ as in \eqref{eq:opt-Rl} and the KL term as in \eqref{eq:D_Rl_diag}, followed by the final generalization bounds.

The designs of $\mA_l$'s should satisfy the perturbation conditions imposed in \eqref{eq:perb-bound} and \eqref{eq:perb-cond}. At the same time, as $\sigma^2$ is a function of $\{\mA_l\}_{l=1}^d$, which involves $\{\Norm{\mW_l}_2\}_{l=1}^d$, we need to adjust it with the approximated $\{\Norm{\hat{\mW}_l}_2\}_{l=1}^d$.

\subsection{Diagonal Matrices}
\label{sec:diag-matrix}
By letting $\mA_l$ be a diagonal matrix, we intend to assume the perturbations are independent per-layer (or even per-parameter), each layer contributes separately to output sensitivity, and there's no cross-talk or interaction between different parameter blocks. This matches the layerwise Lipschitz structure, where the change in output is dominated by per-layer norms and no mixing across dimensions is modeled.

\begin{theorem}\label{thm:diag-matrix}
Consider a general feedforward neural network $f_{\vw}: \gX \to \R^K$ with ReLU activations, depth $d$, width $h$, and bounded inputs $\Norm{\vx}_2 \le B$.
    By setting the sensitivity matrix as
    \begin{align} \label{eq:diagonal-Al}
        \mA_l =  \frac{ e\sqrt{d}B \prod_l \Norm{\mW_l}_2}{\Norm{\mW_l}_2} \mI, \ \forall~l,
    \end{align}
    for any $\delta,\gamma>0$, with probability at least $1-\delta$, for any $\vw$, the generalization error over a training set with size $m$ can be upper-bounded by
    \begin{align}
    L_0(f_{\vw}) \le \hat{L}_{\gamma}(f_{\vw}) + \gO \left( \sqrt{ \frac{ B^2 d^2h^2\Phi(\vw) + \ln \frac{dm}{\delta}}{\gamma^2m}} \right)
    \end{align}
    where $\Phi(\vw)=\prod_l \Norm{\mW_l}_2^2  \sum_{l=1}^d \frac{\Norm{\mW_l}_F^2}{\Norm{\mW_l}_2^2}$.
\end{theorem}

\begin{proof}
See Appendix \ref{proof:diag-matrix}.
\end{proof}

\begin{remark} \upshape
    In contrast to the isotropic Gaussian posterior in \cite{neyshabur2018pac}, the diagonal sensitivity matrix design gives rise to a diagonal optimized posterior covariance matrix $\mR_l^*$ as in \eqref{eq:diagonal-Rl}, where the intra- and inter-layer correlation of model weights are fully decoupled with only different scales across layers. Hence, the obtained generalization error bounds are also similar yet with different scaling on $d$ and $h$. In fact, the choice of $\mA_l$ in \eqref{eq:diagonal-Al} would end up with a looser bound because the perturbation condition \eqref{eq:perb-bound} is validated by the perturbation bound in \eqref{eq:pert-bound-origin}. Nevertheless, with the optimized choice of $\mR_l$ in the unified framework, we end up with a generalization error bound comparable to that in \eqref{eq:gen-bound-origin}.
    Compared with the bound in \cite{neyshabur2018pac}, it appears our bound scales similarly as $\Phi(\vw)$, being tighter for deep networks with $d^2$ versus $d^2 \ln d$, yet looser for wide networks with $h^2$ versus $h\ln h$. 
    Although it is simple and always valid, it incurs the worst scaling of $h^2$, ignoring structure, correlations, and weight sharing.

    With such a diagonal (or isotropic) sensitivity matrix in \eqref{eq:diagonal-Al}, weight perturbations are applied independently and equally in all parameter directions. This corresponds to enforcing uniform flatness of the loss landscape in every coordinate direction, i.e., the loss must not increase significantly under small independent perturbations of each parameter. Such a uniform flatness favors globally flat minima in the standard Euclidean sense, which matches the classical ``flat minima'' intuition, i.e., small Hessian eigenvalues in all directions. However, it would be overly conservative, as it penalizes sharpness even in directions that have little effect on the output, and ignores correlations and task-relevant subspaces.
\end{remark}

\begin{corollary}\label{cor:diag-matrix-ResNet}
    Consider a residually connected feedforward neural network $f_{\vw}: \gX \to \R^K$ with ReLU activations, depth $d$, width $h$, and bounded inputs $\Norm{\vx}_2 \le B$, where the $l$-th residual connection layer is $f_{\vw}^l(\vx) = \mW_l \phi(f_{\vw}^{l-1}(\vx)) + f_{\vw}^{l-1}(\vx)$.
    By setting the sensitivity matrix as
    \begin{align} \label{eq:diagonal-Al-residual}
        \mA_l =  \frac{ e\sqrt{d}B \prod_l (\Norm{\mW_l}_2+1)}{\Norm{\mW_l}_2+1} \mI, \ \forall~l,
    \end{align}
    for any $\delta,\gamma>0$, with probability at least $1-\delta$, for any $\vw$, the generalization error over a training set with size $m$ can be upper-bounded by
    \begin{align}
    L_0(f_{\vw}) \le \hat{L}_{\gamma}(f_{\vw}) + \gO \left( \sqrt{ \frac{ B^2 d^2h^2\Phi^{\mathrm{rn}}(\vw) + \ln \frac{dm}{\delta}}{\gamma^2m}} \right)
    \end{align}
    where $\Phi^{\mathrm{rn}}(\vw)=\prod_l (\Norm{\mW_l}_2+1)^2  \sum_{l=1}^d \frac{\Norm{\mW_l}_F^2}{(\Norm{\mW_l}_2+1)^2}$.
\end{corollary}

\begin{proof}
See Appendix \ref{proof:diag-matrix-ResNet}.
\end{proof}
\begin{remark}\upshape
    For residual networks, the resulting sensitivity matrix takes a form closely related to the diagonal case, with the key difference that each layer contributes a factor of $\Norm{\mW_l}_2+1$ rather than $\Norm{\mW_l}_2$ to the spectral complexity term. Consequently, the generalization bound differs from that of standard feedforward networks only through the replacement of the spectral norm by its shifted counterpart, reflecting the presence of identity skip connections. This structure reveals several properties: (i) perturbations propagate additively along residual paths instead of being purely multiplicative, preventing exponential amplification across depth; (ii) each residual branch contributes largely independently to the overall sensitivity, consistent with the modular design of residual blocks; and (iii) cross-layer amplification of perturbations is effectively suppressed by the identity mapping, yielding improved stability without requiring explicit spectral compression of weights. Notably, this bound is architecture-aware---it exploits the intrinsic modularity and skip-connectivity of residual networks. Thus, the bound provides a theoretical explanation for the empirical robustness and favorable generalization of residual architectures. That is, stability is achieved structurally, through additive identity pathways, rather than solely through norm control of individual layers.
\end{remark}

\subsection{Low-Rank Matrices}
The motivation for designing a low-rank sensitivity matrix stems from the observation that, for most neural networks, the output dimension $K$ is much smaller than the hidden dimension $h$, implying that perturbations in the parameter space can affect the network output only through a low-dimensional subspace. Consequently, controlling weight perturbations uniformly in all $h^2$ parameter directions—as is implicitly done by diagonal sensitivity in the layerwise spectral-norm–based analyses such as those in \cite{neyshabur2018pac}—is overly conservative, unnecessarily leading to looser generalization bounds.

\begin{theorem}\label{thm:low-rank-matrix}
Consider a feedforward neural network $f_{\vw}: \gX \to \R^K$ with ReLU activations, depth $d$, width $h$, bounded inputs $\Norm{\vx}_2 \le B$, and a potential low-rank structure.
    By setting a low-rank sensitivity matrix as 
\begin{align} \label{eq:low-rank-Al}
    \mA_l = \sqrt{d} \mV_l \mathrm{diag}(\sigma_{l,1},\dots,\sigma_{l,K},0,\dots,0)\mV_l^T
\end{align}
with $K\ll h^2$, where $\mV_l$ is the collection of right eigenvectors of the Jacobian of the output with respect to the weights at the $l$-th layer, i.e., $\mJ_l(\vx) = \frac{\partial f_{\vw}(\vx)}{\partial \mathrm{vec}(\mW_l)} \in \R^{K \times h^2}$, and
\begin{align}
    \sigma_{l,k} = B\frac{\prod_{l=1}^d \Norm{\mW_l}_2}{\Norm{\mW_l}_2}, \ \forall~k,
\end{align}
 for any $\delta,\gamma>0$, with probability at least $1-\delta$, for any $\vw$, the generalization error over a training set with size $m$ can be upper-bounded by
    \begin{align}
    L_0(f_{\vw}) \le \hat{L}_{\gamma}(f_{\vw}) + \gO \left( \sqrt{ \frac{B^2d^2K \Phi(\vw) + \ln \frac{dm}{\delta}}{\gamma^2m}} \right).
    \end{align}
\end{theorem}
\begin{proof}
See Appendix \ref{proof:low-rank-matrix}.
\end{proof}
\begin{remark} \upshape
By explicitly leveraging the output bottleneck, a low-rank sensitivity design of $\mA_l$ in \eqref{eq:low-rank-Al} restricts perturbations to the $K$-dimensional subspace that is actually relevant for the output. It, therefore, yields generalization bounds that scale with $K$ rather than $h^2$, which are strictly tighter than the existing results in \cite{neyshabur2018pac} with $K \le h$ for most practical networks. 

The design of low-rank $\mA_l$ is naturally characterized through the Jacobian matrix $\mJ_l(\vx)$, which captures the local sensitivity of the network output with respect to the weights at layer $l$. The dominant eigenvectors in $\mV_l$ identify the directions in parameter space along which perturbations most strongly influence the output. By aligning the sensitivity matrix 
$\mA_l$ with the principal subspace of $\mJ_l(\vx)$, the perturbation is concentrated on the most influential directions, matching the intrinsic sensitivity structure of the network. As such, the effective Jacobian rank or the output dimension $K$ replaces $h^2$ in the model complexity, achieving tighter, task-aligned generalization bounds.

Low-rank sensitivity restricts perturbations to a subspace aligned with the row space of the Jacobian, where only the directions that actually affect the network outputs count. It enforces flatness only along output-relevant directions, allowing for possible sharp curvature in null directions that do not affect predictions.
From a loss landscape viewpoint, the generalization can still be good as long as the loss is flat in the prediction-sensitive subspace, even if the entire model may lie in a sharp basin globally. This aligns with the understanding that not all sharpness is harmful---only the output-relevant sharpness is.  
\end{remark}

\subsection{Circulant Matrices}
For fully convolutional models, parameters exhibit strong global structure induced by weight sharing and local translation invariance, which cannot be faithfully captured by diagonal or low-rank sensitivity designs. By choosing the sensitivity matrix $\mA_l$ to be circulant, we explicitly model correlations across adjacent parameters and emulate convolution-like interactions in the parameter space. More interestingly, by taking the Fourier transform of the weight perturbations, we can conduct perturbation analysis in the frequency domain. 

\begin{theorem} \label{thm:circulant-matrix}
    Consider a fully convolutional network model $f_{\vw}: \gX \to \R^K$ with ReLU activations, depth $d$, width $h$, and bounded inputs $\Norm{\vx}_2 \le B$. The weight matrix $\mW_l$ is comprised of convolutional filters $\vw_l \in \R^h$ with circulant padding at the $l$-th layer. By setting the sensitivity matrix as
    \begin{align}
    \mA_l = \sqrt{d} B \mV \mathrm{diag}(\underbrace{\lambda_l,\dots,\lambda_l,}_{K \text{ times}} 0,\dots,0) \mV^H
\end{align}
where $\lambda_l=\prod_{i\ne l} \Norm{\mV^H \vw_i}_{\infty}$
    for all $l$ with $\mV$ being the $h \times h$ normalized discrete Fourier transform (DFT) matrix, for any $\delta,\gamma>0$, with probability at least $1-\delta$, for any $\vw=(\vw_1,\dots,\vw_d)$, the generalization error over a training set with size $m$ can be upper-bounded by
    \begin{align}
    L_0(f_{\vw}) \le \hat{L}_{\gamma}(f_{\vw}) + \gO \left( \sqrt{ \frac{ B^2d^2K {\Phi}^{\mathrm{circ}}(\vw) + \ln \frac{dm}{\delta}}{\gamma^2m}} \right)
    \end{align}
with ${\Phi}^{\mathrm{circ}}(\vw)=\prod_{i\ne l} \Norm{\mV^H \vw_l}_{\infty}^2 \sum_{l=1}^d \Norm{\vw_l}_2^2$.
\end{theorem}

\begin{proof}
See Appendix \ref{proof:circulant-matrix}.
\end{proof}

\begin{remark}\upshape
For the convolutional layers in practical convolutional neural networks (CNNs), the weight matrix $\mW_l$ is a circulant matrix, i.e., $\mW_l = \mathrm{circ}(\vw_l)$, with the convolutional kernel $\vw_l$ circularly repeated across rows in $\mW_l$ according to cyclic convolution and padding \cite{sedghi2018the,gray2006toeplitz}. For fully CNNs, we have a new spectral complexity measure ${\Phi}^{\mathrm{circ}}(\vw)$, where the circulant weight matrix $\mW_l$ in the parameter space is replaced with the Fourier-transformed convolutional filter, i.e., $\mV^H \vw_l$, in the frequency domain. Weight sharing (cf. repeated kernel vectors with cyclic convolution) is also taken into account, where $\Norm{\mW_l}_F^2$ is replaced by  $\Norm{\vw_l}_2^2$. Such a spectral complexity is more aligned with convolutional architectures, and is substantially smaller than ${\Phi}(\vw)$, yielding significantly tighter generalization bounds.
It also demonstrates that architectural priors, i.e., translation equivariance and weight sharing, can be directly embedded into PAC-Bayesian generalization bounds via structured sensitivity matrices, yielding bounds that are not only sharper but also more interpretable from a frequency-domain signal processing viewpoint.

As $\mV$ is the DFT matrix, the sensitivity matrix $\mA_l$ is also circulant \cite{gray2006toeplitz}, aligning with the circulant weight matrix $\mW_l$. This allows for the sensitivity analysis in the frequency domain, where the Fourier-transformed weight perturbations $\mV^H \vu_l$ are added to $\mV^H \vw_l$ in the frequency domain, so that weight perturbations are regulated by $\lambda_l$ according to their frequency responses. 
The circulant $\mA_l$ implies that posterior perturbations are coupled through cyclic convolution, inducing global yet highly structured correlations among parameters, i.e., weight sharing. In the frequency domain, the gains of $\mA_l$, i.e., $\lambda_l$ for the first $K$ dimensions, directly modulate the posterior variance $\mR_l$ as in \eqref{eq:low-rank-Rl} across frequencies. That is, frequencies with larger gain in $\mA_l$ correspond to directions where the network output is more sensitive, and the posterior variance $\mR_l$ is therefore shrunk more aggressively to suppress harmful perturbations. This yields an anisotropic and frequency-selective perturbation model, so that stability can be enforced where it matters most, while ignoring insensitive frequency bands. 
    
With the circulant sensitivity, weight perturbation favors flatness per frequency band rather than per parameter.
It enforces flatness in frequency bands where sensitivity is high, and suppresses harmful high-frequency perturbations.
Notably, the loss landscape is encouraged to be flat with respect to structured, convolution-like perturbations rather than arbitrary ones. This connects nicely to observations that CNNs generalize by suppressing high-frequency filters and noise.
\end{remark}

\subsection{Toeplitz Matrices}
To model convolutional weight sharing without imposing artificial periodicity (i.e., circulant padding), we consider Toeplitz sensitivity matrices, which naturally encode locality and distance-dependent correlations among parameters, reflecting the fact that perturbations of a single kernel coefficient affect many spatial locations.

\begin{theorem} \label{thm:toep-matrix}
    Consider a linear convolutional network $f_{\vw}: \gX \to \R^K$ with ReLU activations, depth $d$, width $h$, and bounded inputs $\Norm{\vx}_2 \le B$. At the $l$-th linear convolutional layer, the weight matrix $\mW_l \in \R^{h \times h}$ is a Toeplitz matrix with kernel $\vw_l \in \R^k$ such that $\mW_l=\mathrm{toep}(\vw_l)$. By setting the sensitivity matrix as
    \begin{align}
        {\mA}_l =  \frac{ e\sqrt{d}B  \prod_{i\ne l} \Norm{\mW_i}_2}{\psi_{\min}}\mT \mP
    \end{align}
    for all $l$ with $\mT \in \R^{h^2 \times h^2}$ being a Toeplitz matrix with symbol $\psi(\omega) \in [\psi_{\min},\psi_{\max}]$ for all $\omega \in [0,2\pi]$ and $\mP \in \R^{h^2 \times k}$ a structured sparse matrix such that $\mathrm{vec}({\mathrm{toep}(\vw_l)})=\mP \vw_l$, for any $\delta,\gamma>0$, with probability at least $1-\delta$, for any $\vw$, the generalization error over a training set with size $m$ can be upper-bounded by 
\begin{align}
    L_0(f_{\vw}) \le \hat{L}_{\gamma}(f_{\vw}) + \mathcal{O} \left( \sqrt{ \frac{B^2d^2k \Phi^{\mathrm{toep}}(\vw) + \ln \frac{md}{\delta}}{\gamma^2m}} \right)
\end{align}
where the kernel spectral complexity is given by
\begin{align}
    \Phi^{\mathrm{toep}}(\vw)=\frac{\psi_{\max}^2}{\psi_{\min}^2} \prod_{i\ne l} \Norm{\vw_i}_1^2 \sum_{l=1}^d \Norm{\vw_l}_2^2. 
\end{align}
\end{theorem}
\begin{proof}
    See Appendix \ref{proof:toep-matrix}.
\end{proof}
\begin{remark} \upshape
For the linear convolutional layers, the efficient weight matrix $\mW_l$ is a Toeplitz matrix consisting of repeated Toeplitz kernels without circulant padding \cite{yi2022asymptotic,Wang_2020_CVPR}.
Similar to the circulant case, we have a newly defined ``kernel'' spectral complexity ${\Phi}^{\mathrm{toep}}(\vw)$, which is a function of the Toeplitz kernel, i.e., $\vw_l$.
Nevertheless, unlike circulant matrices, Toeplitz sensitivity captures convolution with finite support and no wrap-around effects, making it more faithful to practical convolutional architectures. In this setting, weight perturbations $\{\vu_l\}_{l=1}^d$, added directly to weight kernels $\{\vw_l\}_{l=1}^d$, are correlated across parameters through $\mP$ according to spatial proximity, enforcing smoothness and locality in the posterior geometry. As a result, the resulting PAC-Bayesian bound replaces weight matrices with the Toeplitz kernel vectors, adjusted by the dynamic range of the filter’s frequency response, i.e., $\frac{\psi_{\max}^2}{\psi_{\min}^2}$.
This change reflects that the effective complexity is governed by the bandwidth $k$ and shape of the convolutional kernel, rather than by the total parameter dimension.

The Toeplitz sensitivity design enables a precise spectral alignment between the posterior covariance $\mR_l$, the intrinsic sensitivity structure $\mA_l$, and the PAC-Bayesian KL geometry. In contrast to diagonal sensitivity, which enforces isotropic shrinkage in the parameter space, Toeplitz structure induces anisotropic and frequency-dependent shrinkage through its symbol $\psi(\omega)$ with $\omega \in [0,2\pi]$, allowing the posterior $\mR_l$ as in \eqref{eq:toeplitz-Rl} to contract more strongly in spectrally sensitive directions while preserving stable, low-frequency components. 
Toeplitz sensitivity naturally encodes weight sharing and locality priors, effectively reducing the relevant spectral dimension even when the parameter dimension remains unchanged. As a result, the induced posterior becomes frequency-selective, leading to interpretable smoothness regularization that suppresses high-frequency, spatially incoherent perturbations while retaining low-frequency, structure-preserving variations. 

Toeplitz sensitivity enforces flatness with respect to local, smooth deformations of filters. It encourages flatness under spatially local and coherent perturbations, while allowing for sharpness under highly unstructured perturbations.
In other words, Toeplitz sensitivity prefers minima that are flat with respect to local correlated directions, ignoring possibly sharp in highly irregular directions that violate convolutional structures.
\end{remark}

\section{Conclusion}
This work presented a unified PAC-Bayesian norm-based generalization framework for deep learning by reformulating the derivation of generalization bounds as a stochastic optimization problem over anisotropic Gaussian posteriors. The proposed framework explicitly accounts for the anisotropic posterior covariances and network architectural structures, unifying and extending a broad class of existing PAC-Bayesian results that were derived under restrictive isotropic posterior assumptions. The key novelty was the introduction of structure-aware sensitivity matrices with e.g., diagonal, low-rank, circulant, and Toeplitz structures, which enable fine-grained control of weight perturbation effects, leading to theoretically sharper and more interpretable generalization bounds. 
Such a framework is unified in the sense that it can be extended naturally to different settings, e.g., the adversarial settings, where the interaction between input perturbations and weight perturbations poses additional challenges for robust generalization, and graph neural networks, where weight sharing and graph operators call for structure-aware and geometry-sensitive designs. The proposed unified framework could bring in different perspectives for addressing these challenges and yield tighter and more interpretable PAC-Bayesian generalization bounds.

\section{Appendix}
\subsection{Proof of Theorem \ref{thm:diag-matrix}}
\label{proof:diag-matrix}
The proof follows those in the main text with the following special considerations:
\begin{itemize}
    \item The perturbation condition is satisfied by a looser perturbation bound than that in \eqref{eq:pert-bound-origin}.
    \item The concentration inequality on $\ell_2$ norms yields the final generalization bound exempting from a $\ln (dh)$ factor.
\end{itemize}
The key steps are as follows.
\subsubsection{Perturbation condition}
When $\mA_l = \lambda_l \mI$ with
\begin{align}
        \lambda_l =  \frac{ e\sqrt{d}B \prod_l \Norm{\mW_l}_2}{\Norm{\mW_l}_2}, \ \forall~l,
    \end{align}
we have
\begin{align}
    \Norm{\mA \vu}_2^2 = \sum_{l=1}^d\Norm{\mA_l \vu_l}_2^2 &= e^2B^2 d\prod_l \Norm{\mW_l}_2^2  \sum_{l=1}^d \frac{\Norm{\mU_l}_F^2}{\Norm{\mW_l}_2^2} \\
    &\ge e^2B^2 \prod_l \Norm{\mW_l}_2^2 \left(\sum_{l=1}^d \frac{\Norm{\mU_l}_2}{\Norm{\mW_l}_2}\right)^2 \nn\\
    &\ge \Norm{f_{\vw+\vu}(\vx)-f_{\vw}(\vx)}_2^2\\
    &\ge \Norm{f_{\vw+\vu}(\vx)-f_{\vw}(\vx)}_\infty^2
\end{align}
where the first inequality is due to the Cauchy-Schwarz inequality, i.e., 
\begin{align}
\left(\sum_{l=1}^d \frac{\Norm{\mU_l}_2}{\Norm{\mW_l}_2}\right)^2 &\le  \left(\sum_{l=1}^d \frac{\Norm{\mU_l}_2^2}{\Norm{\mU_l}_F^2}\right) \left(\sum_{l=1}^d \frac{\Norm{\mU_l}_F^2}{\Norm{\mW_l}_2^2}\right) \\
&\le d \left(\sum_{l=1}^d \frac{\Norm{\mU_l}_F^2}{\Norm{\mW_l}_2^2}\right), 
\end{align}
with $\frac{\Norm{\mU_l}_2^2}{\Norm{\mU_l}_F^2} \le 1$, 
the second inequality is from \cite[Lemma~2]{neyshabur2018pac} for $\Norm{\mU_l}_2 \le \frac{1}{d}\Norm{\mW_l}_2$, and the last one is because $\Norm{\vv}_\infty \le \Norm{\vv}_2$ for any vector $\vv$.

\subsubsection{The choice of $\sigma^2$ for a fixed $\hat{\beta}$}
In the spectrally-normalized setting, for each approximated normalized spectral norm $\hat{\beta}=\Norm{\hat{\mW}_l}_2$ for all $l$ on a pre-determined grid, we compute the norm-based bounds and establish the generalization guarantee for all $\vw$ for which $\Abs{\beta-\hat{\beta}} \le \frac{1}{d}\beta$, to ensure each relevant value of $\beta=\Norm{{\mW}_l}_2$ is covered by some $\hat{\beta}$ in the grid. As such, taking a union bound over all $\hat{\beta}$ on the grid will help us address the issue of the required weights in the prior. In doing so, we can consider a fixed $\hat{\beta}$ for which $\Abs{\beta-\hat{\beta}} \le \frac{1}{d}\beta$ in setting the prior and deriving the bounds, by simply introducing an approximation of $\mA_l$ for a fixed $\hat{\beta}$, i.e.,
\begin{align}
    \hat{\mA}_l =  \frac{e\sqrt{d}B \prod_l \Norm{\hat{\mW}_l}_2  }{\Norm{\hat{\mW}_l}_2} \mI.
\end{align}

When $\Abs{\beta-\hat{\beta}} \le \frac{1}{d}\beta$, it follows $\frac{1}{e} \beta^{d-1} \le \hat{\beta}^{d-1} \le e\beta^{d-1}$ for all $d \ge 1$ \cite{neyshabur2018pac}, so that
\begin{align} \label{eq:Al-hatAl}
   \frac{1}{e} \mA_l \preceq \hat{\mA}_l \preceq e \mA_l.
\end{align}
Letting
\begin{align}
     \mR_l &= (\mI + \eta^2 {\mA}_l^T {\mA}_l)^{-1}, 
\end{align}
with $\eta^2=\frac{16   \kappa \Norm{\vw}_2^2}{\gamma^2}$, we have
\begin{align}
    \Tr({\mA}_l {\mR}_l {\mA}_l^T) \le \Tr({\mA}_l {\mA}_l^T) \le e^2\Tr(\hat{\mA}_l \hat{\mA}_l^T)
\end{align}
where the first inequality holds because $\mR_l \preceq \mI$ and the second inequality is due to \eqref{eq:Al-hatAl}.

To satisfy the perturbation condition, we need to ensure, by a proper choice of $\sigma^2$, with probability at least $\frac{1}{2}$, that the following inequality holds
\begin{align}
   \sum_{l=1}^d \Norm{{\mA}_l \vu_l}_2^2 &\le \sigma^2 \kappa \sum_{l=1}^d \Tr({\mA}_l {\mR}_l {\mA}_l^T)\\
   &\le e^2 \sigma^2 \kappa \sum_{l=1}^d \Tr(\hat{\mA}_l \hat{\mA}_l^T)
   \le  \frac{\gamma^2}{16}
\end{align}
 
By simply setting
\begin{align} \label{eq:choice-sigma-diag}
    \frac{1}{{\sigma}^2} &= \frac{16e^2 \kappa}{\gamma^2} \sum_{l=1}^d \Tr(\hat{\mA}_l  \hat{\mA}_l^T)
\end{align}
we end up with the final choice ${\sigma}^2$ as a function of $\hat{\beta}$.

Note also that the use of \cite[Lemma~2]{neyshabur2018pac} for perturbation condition requires $\Norm{\mU_l}_2 \le \frac{1}{d}\Norm{\mW_l}_2 = \frac{\beta}{d}$ for all $l$. This can be satisfied if
\begin{align}
    \sigma^2 \le \frac{\beta^2}{\kappa d^2 h^2}
\end{align}
due to
\begin{align}
    e^2 B^2 \beta^{2d-2} \left( \sum_{l=1}^d\Norm{\mU_l}_2\right)^2 &\le \Norm{\mA \vu}_2^2 \\
    &\le \sigma^2 \kappa \sum_{l=1}^d \Tr(\mA_l \mR_l \mA_l^T)\\
    & \le \sigma^2 \kappa \sum_{l=1}^d \Tr(\mA_l \mA_l^T)\\
    & \le e^2 B^2 \beta^{2d}.
\end{align}
It turns out to require 
\begin{align}
    \frac{1}{{\sigma}^2} &= \frac{16e^2 \kappa}{\gamma^2} \sum_{l=1}^d \Tr(\hat{\mA}_l  \hat{\mA}_l^T) \\
    & =  \frac{16e^2 \kappa}{\gamma^2} e^2 B^2 d^2 h^2 \hat{\beta}^{2d-2} \\
    & \ge \frac{16e^2 \kappa}{\gamma^2} B^2 d^2 h^2 {\beta}^{2d-2} \\
    &\ge \frac{\kappa d^2 h^2}{ \beta^2}
\end{align}
which requires $\beta^d \ge \frac{\gamma}{4eB}$. This can be guaranteed later when we consider the nontrivial range of $\beta$.

\subsubsection{PAC-Bayesian bound for a fixed $\hat{\beta}$}
Next, according to \eqref{eq:opt-Rl}, we have the optimized $\mR_l$, i.e.,
\begin{align} \label{eq:diagonal-Rl}
    {\mR}_l^* = \frac{1}{1+{\eta}^2 {\lambda}_l^2} \mI.
\end{align}

With the choice of $\sigma^2$ in \eqref{eq:choice-sigma-diag}, the KL divergence can be upper bounded by
\begin{align}
    \KL & \le \frac{8e^{2} \kappa\Norm{\vw}_2^2}{\gamma^2} \sum_{l=1}^d \Tr(\hat{\mA}_l \hat{\mA}_l^T) \nn \\
    & \qquad \qquad + \frac{1}{2} \sum_{l=1}^d \left(\Tr(\mR_l^*)  - \log \det \mR_l^* - h^2 \right)\\
    &\le \frac{8 e^4 \kappa h^2 \Norm{\vw}_2^2}{\gamma^2 }   \sum_{l=1}^d \lambda_l^2 + \frac{1}{2}\sum_{l=1}^d h^2 \delta({\eta}{\lambda}_l)\\
    &\le \frac{8 (e^4+1) \kappa h^2 \Norm{\vw}_2^2 }{\gamma^2 } \sum_{l=1}^d \lambda_l^2\\
    &\lesssim \gO\left(  \frac{B^2 d^2 h^2\beta^{2d}}{\gamma^2} \sum_{l=1}^d\frac{\Norm{\mW_l}_F^2}{\beta^2}\right) \\
    &\le \gO\left(\frac{B^2d^2h^2}{\gamma^2} \prod_l \Norm{\mW_l}_2^2 \sum_{l=1}^d \frac{\Norm{\mW_l}_F^2}{\Norm{\mW_l}_2^2}\right)\\
    &=\gO\left(\frac{B^2d^2h^2}{\gamma^2} \Phi(\vw)\right)
\end{align}
where 
\begin{align} \label{eq:delta-func}
    \delta(x)\triangleq\log(1+x^2)-\frac{x^2}{1+x^2} \le x^2,
\end{align}
and $\Phi(\vw)=\prod_l \Norm{\mW_l}_2^2  \sum_{l=1}^d \frac{\Norm{\mW_l}_F^2}{\Norm{\mW_l}_2^2}$.

Hence, for any $\hat{\beta}$, with probability at least $1-\delta$, for any $\vw$ such that $\Abs{\beta-\hat{\beta}} \le \frac{1}{d}\beta$, we have
\begin{align}
    L_0(f_{\vw}) \le \hat{L}_{\gamma}(f_{\vw}) + \mathcal{O} \left( \sqrt{ \frac{B^2d^2h^2 \Phi(\vw) + \ln \frac{m}{\delta}}{\gamma^2m}} \right).
\end{align}

\subsubsection{Union bound for any $\beta$}
Finally, following the same arguments in \cite{neyshabur2018pac}, we take a union bound over different choices of $\hat{\beta}$, and end up with the final generalization error bounds. 

Specifically, the nontrivial $\beta$ under consideration should be within the range
\begin{align}
    \left(\frac{\gamma}{2B}\right)^{1/d} \le \beta \le \left(\frac{\gamma \sqrt{m}}{2B})\right)^{1/d}.
\end{align}
Otherwise, either the empirical margin loss $\hat{L}_{\gamma}(f_{\vw})$ or the generalization error is greater than or equal to 1, yielding that the generalization bound trivially holds \cite{neyshabur2018pac,liao2020pac}. On one hand, if $\beta < \left(\frac{\gamma}{2B}\right)^{1/d} $, it follows $\Norm{f_{\vw}(\vx)}_2 \le \beta^d \Norm{\vx} \le \beta^d B \le \frac{\gamma}{2}$, yielding $\hat{L}_{\gamma}(f_{\vw})=1$ by definition. On the other hand, if $\beta > \left(\frac{\gamma \sqrt{m}}{2B}\right)^{1/d} $, it follows that
\begin{align}
    \sqrt{ \frac{B^2d^2h^2 \Phi(\vw) + \ln \frac{m}{\delta}}{\gamma^2m}} &\ge \sqrt{ \frac{d^2h^2 }{4} \sum_{l=1}^d \frac{\Norm{\mW_l}_F^2}{\Norm{\mW_l}_2^2}} \\
    & \ge 1
\end{align}
with $d \ge 2$, $h \ge 2$, and $\Norm{\mW_l}_F \ge \Norm{\mW_l}_2$.

To make $\Abs{\beta-\hat{\beta}} \le \frac{\beta}{d}$ satisfied, we require $\Abs{\beta-\hat{\beta}} \le \frac{1}{d}\left(\frac{\gamma}{2B}\right)^{\frac{1}{d}}$. If a covering of the interval with radius $\frac{1}{d}\left(\frac{\gamma}{2B}\right)^{\frac{1}{d}}$ can make the fixed-$\hat{\beta}$ bound validated with $\hat{\beta}$ taking all possible values from the covering, then we could end up with a bound that holds for all $\beta$. It follows that such a covering exists with size of $dm^{\frac{1}{2d}}$. Taking a union bound over all choices of $\hat{\beta}$ with the cover yields the final generalization error bound
\begin{align}
    L_0(f_{\vw}) \le \hat{L}_{\gamma}(f_{\vw}) + \mathcal{O} \left( \sqrt{ \frac{B^2d^2h^2 \Phi(\vw) + \ln \frac{md}{\delta}}{\gamma^2m}} \right)
\end{align}
for any $\beta$. 
This completes the proof.

\subsection{Proof of Corollary \ref{cor:diag-matrix-ResNet}}
\label{proof:diag-matrix-ResNet}
The proof is similar to that of Theorem \ref{thm:diag-matrix} with the modification for the perturbation condition.
With the choice of 
\begin{align}
    {\mA}_l =  \frac{e\sqrt{d}B \prod_l (\Norm{{\mW}_l}_2+1)  }{\Norm{{\mW}_l}_2+1} \mI,
\end{align}
we have
\begin{align}
    \sum_{l=1}^d\Norm{\mA_l \vu_l}_2^2 &= e^2B^2 d\prod_l (\Norm{\mW_l}_2+1)^2  \sum_{l=1}^d \frac{\Norm{\mU_l}_F^2}{(\Norm{\mW_l}_2+1)^2} \\
    &\ge e^2B^2 \prod_l (\Norm{\mW_l}_2+1)^2 \left(\sum_{l=1}^d \frac{\Norm{\mU_l}_2}{\Norm{\mW_l}_2+1}\right)^2 \nn\\
    &\ge \Norm{f_{\vw+\vu}(\vx)-f_{\vw}(\vx)}_2^2
\end{align}
where the last inequality is due to \cite[Theorem 8]{xiao2023pac}.

Similarly, consider a fixed $\hat{\beta}$ for which $\Abs{\beta-\hat{\beta}} \le \frac{1}{d}\beta$, by simply introducing an approximation of $\mA_l$ for a fixed $\hat{\beta}$, i.e.,
\begin{align}
    \hat{\mA}_l =  \frac{e\sqrt{d}B \prod_l (\Norm{\hat{\mW}_l}_2+1)  }{\Norm{\hat{\mW}_l}_2+1} \mI.
\end{align}
Given $\Abs{\beta-\hat{\beta}} \le \frac{1}{d}\beta$, it holds $\frac{1}{e} \beta^{d-1} \le \hat{\beta}^{d-1} \le e\beta^{d-1}$, according to \cite{neyshabur2018pac}. Let $r=e^{\frac{1}{d-1}} \ge 1$. By simple manipulations, i.e., $\frac{\beta+1}{r} \le \frac{\beta}{r}+1 \le \hat{\beta}+1 \le r \beta + 1 \le r(\beta+1) $, we have
\begin{align}
    \frac{1}{e} (\beta+1)^{d-1} \le (\hat{\beta}+1)^{d-1} \le e(\beta+1)^{d-1}
\end{align}
which yields the same relation between $\mA_l$ and $\hat{\mA}_l$, i.e.,
\begin{align}
   \frac{1}{e} \mA_l \preceq \hat{\mA}_l \preceq e \mA_l.
\end{align}
Therefore, by setting
\begin{align} \label{eq:choice-sigma-diag-ResNet}
    \frac{1}{{\sigma}^2} &= \frac{16e^2 \kappa}{\gamma^2} \sum_{l=1}^d \Tr(\hat{\mA}_l  \hat{\mA}_l^T)
\end{align}
and replacing the previous consideration of $\beta$ with $\beta+1$, we end up with the final generalization error bound
\begin{align}
    L_0(f_{\vw}) \le \hat{L}_{\gamma}(f_{\vw}) + \mathcal{O} \left( \sqrt{ \frac{B^2d^2h^2 \Phi^{\mathrm{rn}}(\vw) + \ln \frac{md}{\delta}}{\gamma^2m}} \right)
\end{align}
for any $\beta$. 
This completes the proof.

\subsection{Proof of Theorem \ref{thm:low-rank-matrix}}
\label{proof:low-rank-matrix}
The procedure is similar to that in the diagonal case, while the difference lies in the following two points:
\begin{itemize}
    \item Without relying on the perturbation bound in \eqref{eq:pert-bound-origin}, the difference of the network's output due to weight perturbation is approximated by its first-order Taylor expansion with the Jacobian $\mJ_l(\vx)$.
    \item The sensitivity matrix $\mA_l$ captures the directions aligning with the right singular vectors of $\mJ_l(\vx)$, so that the scaling factor of generalization bound reduces from $h^2$ to $K$.
\end{itemize}
\subsubsection{Perturbation condition}
The perturbation difference is the change in the output $f_{\vw}(\vx)$ when the weights are perturbed by $\vu$. It is influenced by the Jacobian matrix $\mJ_l(\vx)$ at each layer, which determines how sensitive the output is to perturbations in the weights.

By the first-order Taylor expansion of $f_{\vw}(\vx)$ with a small perturbation $\vu$, we have
\begin{align}
    f_{\vw+\vu}(\vx) - f_{\vw}(\vx) = \sum_{l=1}^d \mJ_l(\vx) \vu_l  + o(\Norm{\vu_l}_2)
\end{align}
where $\mJ_l(\vx)=\frac{\partial f_{\vw}(\vx)}{\partial \mathrm{vec}(\mW_l)} \in \R^{K \times h_lh_{l-1}} $ is the Jacobian matrix corresponding to the $l$-th layer.
Therefore, we have
\begin{align}
    \Norm{f_{\vw+\vu}(\vx) - f_{\vw}(\vx)}_2^2 &= \Norm{\sum_{l=1}^d \mJ_l(\vx) \vu_l}_2^2 + o(\Norm{\vu_l}_2^2)\\
    &\le d \sum_{l=1}^d \Norm{\mJ_l(\vx) \vu_l}_2^2.
\end{align}

When considering ReLU networks with bounded inputs (i.e., $\Norm{\vx}_2 \le B$), we have
\begin{align}
    \sigma_{l,\max} = \Norm{\mJ_l(\vx)}_2 \le B\frac{\prod_{l=1}^d \Norm{\mW_l}_2}{\Norm{\mW_l}_2}.
\end{align}
This is because the Jacobian $\mJ_l(\vx)$ is influenced by both the forward pass, i.e., from the input to layer-$l$, and the backward pass, i.e., from layer-$l$ to the output, as well as the fact that the ReLU activation function is 1-Lipschitz. Therefore, the spectral norm of $\mJ_l(\vx)$ is upper-bounded by the product of the spectral norms of the weight matrices of other layers, multiplied by the upper bound of the input norm.

Let 
\begin{align}
    \mA_l = \sqrt{d} \mV_{l} \mSigma_{l} \mV_{l}^T
\end{align}
where $\mV_l$ is the collections of corresponding right singular vectors of $\mJ_l$, and $\mSigma_l=\mathrm{diag} (\sigma_{l,1},\dots,\sigma_{l,K}, 0, \dots,0)$ with
\begin{align}
    \sigma_{l,k} = B\frac{\prod_{l=1}^d \Norm{\mW_l}_2}{\Norm{\mW_l}_2}, \quad 
    \forall\; k=1,\dots,K.
\end{align}
It is clear that 
\begin{align}
\mJ_l^T \mJ_l \preceq \mV_{l} \mSigma_{l}^2 \mV_{l}^T = \frac{1}{d}\mA_l^T\mA_l.
\end{align}
Therefore, we have
\begin{align}
    \sum_{l=1}^d \Norm{\mA_l \vu_l}_2^2 &= \sum_{l=1}^d\vu_l^T \mA_l^T \mA_l \vu_l \\
    &=d \sum_{l=1}^d\vu_l^T \mJ_l^T(\vx) \mJ_l(\vx) \vu_l = d \sum_{l=1}^d \Norm{\mJ_l(\vx) \vu_l}_2^2 \nn\\
    &\ge \Norm{f_{\vw+\vu}(\vx)-f_{\vw}(\vx)}_2^2\\
     & \ge \Norm{f_{\vw+\vu}(\vx)-f_{\vw}(\vx)}_\infty^2.
\end{align}
for any $\vu_l$, so that the perturbation condition \eqref{eq:pert-cond} is satisfied.

\subsubsection{The choice of $\sigma$ of a fixed $\hat{\beta}$}
Similarly to the diagonal case in Section \ref{sec:diag-matrix}, for a fixed $\hat{\beta}$, we introduce the approximation of $\mA_l$, i.e., $\hat{\mA}_l = \sqrt{d} \mV_l \mathrm{diag}(\hat{\sigma}_{l,1},\dots,\hat{\sigma}_{l,K},0,\dots,0)\mV_l^T$, where
\begin{align}
    \hat{\sigma}_{l,k} = B\frac{\prod_{l=1}^d \Norm{\hat{\mW}_l}_2}{\Norm{\hat{\mW}_l}_2}, \quad \forall\; k=1,\dots,K.
\end{align}

As it holds $\frac{1}{e} \beta^{d-1} \le \hat{\beta}^{d-1} \le e\beta^{d-1}$ when $\Abs{\beta-\hat{\beta}} \le \frac{1}{d}\beta$ \cite{neyshabur2018pac}, we have
\begin{align}
   \frac{1}{e} \mA_l \preceq \hat{\mA}_l \preceq e \mA_l.
\end{align}
Letting
\begin{align} \label{eq:low-rank-Rl}
    {\mR}_l  = (\mI+\eta^2 \mA_{l}^T \mA_{l} )^{-1}
\end{align}
with $\eta^2=\frac{16   \kappa \Norm{\vw}_2^2}{\gamma^2}$, we have
\begin{align}
    \Tr({\mA}_l {\mR}_l {\mA}_l^T) \le \Tr({\mA}_l {\mA}_l^T) \le e^2\Tr(\hat{\mA}_l \hat{\mA}_l^T)
\end{align}
where the first inequality holds because $\mR_l \preceq \mI$.
To satisfy the perturbation condition, we need to ensure that
\begin{align}
   \sum_{l=1}^d \Norm{{\mA}_l \vu_l}_2^2 &\le \sigma^2 \kappa \sum_{l=1}^d \Tr({\mA}_l {\mR}_l {\mA}_l^T)\\
   &\le e^2 \sigma^2 \kappa \sum_{l=1}^d \Tr(\hat{\mA}_l \hat{\mA}_l^T)
   \le  \frac{\gamma^2}{16}
\end{align}
 
By simply setting
\begin{align} \label{eq:sigma-upper-low-rank}
    \frac{1}{{\sigma}^2} &= \frac{16e^2 \kappa}{\gamma^2} \sum_{l=1}^d \Tr(\hat{\mA}_l  \hat{\mA}_l^T)=\frac{16e^2 \kappa d}{\gamma^2} \sum_{l=1}^d \sum_{k=1}^K \hat{\sigma}^2_{l,k}
\end{align}
we end up with the final choice ${\sigma}^2$ as a function of $\hat{\beta}$.

\subsubsection{PAC-Bayesian bound for a fixed $\hat{\beta}$}
Further, for simplicity, we let $h_l=h$ for all $l$. With the optimized $\mR_l$, i.e.,
\begin{align} \label{eq:low-rank-Rl-opt}
    {\mR}_l^*  &= (\mI+\eta^2 d\mV_{l} \mSigma_{l}^2 \mV_{l}^T )^{-1} \\
    &=  \mV_{l} (\mI + \eta^2 d\mSigma_{l}^2)^{-1} \mV_{l}^T, 
\end{align}
plugging it into the KL term  with the choice of $\frac{1}{\sigma^2}$ in \eqref{eq:sigma-upper-low-rank}, we have
\begin{align}
\KL & \le \frac{8e^{2} \kappa\Norm{\vw}_2^2}{\gamma^2} \sum_{l=1}^d \Tr(\hat{\mA}_l \hat{\mA}_l^T) \nn \\
    & \qquad \qquad + \frac{1}{2} \sum_{l=1}^d \left(\Tr(\mR_l^*)  - \log \det \mR_l^* - h^2 \right)\\
    &\le \frac{8 e^2 \kappa d \Norm{\vw}_2^2}{\gamma^2 }   \sum_{l=1}^d \sum_{k=1}^K \hat{\sigma}_{l,k}^2 +\frac{d}{2} \sum_{l=1}^d \sum_{k=1}^K \delta({\eta} \sigma_{l,k})\\
    &\le \frac{8 e^4 \kappa d \Norm{\vw}_2^2}{\gamma^2 }   \sum_{l=1}^d \sum_{k=1}^K {\sigma}_{l,k}^2 +\frac{d\eta^2}{2} \sum_{l=1}^d \sum_{k=1}^K \sigma^2_{l,k}\\
    &\le \frac{8 (e^4+1) \kappa d  \Norm{\vw}_2^2 }{\gamma^2 } \sum_{l=1}^d \sum_{k=1}^K {\sigma}_{l,k}^2\\
    &\lesssim \gO\left(  \frac{B^2 d^2 K\beta^{2d}}{\gamma^2} \sum_{l=1}^d\frac{\Norm{\mW_l}_F^2}{\beta^2}\right) \\
    &=\gO\left(\frac{B^2d^2K}{\gamma^2} \Phi(\vw)\right)
\end{align}
where $\delta(\cdot)$ is defined in \eqref{eq:delta-func} and $\Phi(\vw)$ is defined as before.
This yields the generalization error bound
\begin{align}
    L_0(f_{\vw}) \le \hat{L}_{\gamma}(f_{\vw}) + \mathcal{O} \left( \sqrt{ \frac{B^2d^2K \Phi(\vw) + \ln \frac{m}{\delta}}{\gamma^2m}} \right)
\end{align}
for a fixed $\hat{\beta}$.

With the same arguments as in \cite{neyshabur2018pac}, we end up with the final generalization bound in the theorem 
\begin{align}
    L_0(f_{\vw}) \le \hat{L}_{\gamma}(f_{\vw}) + \mathcal{O} \left( \sqrt{ \frac{B^2d^2K \Phi(\vw) + \ln \frac{md}{\delta}}{\gamma^2m}} \right)
\end{align}
for any ${\beta}$, 
where $K \le h$ for most neural networks. This bound is strictly tighter than the previous one \cite{neyshabur2018pac}. This completes the proof.

\subsection{Proof of Theorem \ref{thm:circulant-matrix}}
\label{proof:circulant-matrix}
Let us consider the convolutional neural networks (CNNs) with ReLU activation functions, where $\{\mW_l\}_{l=1}^d \in \R^{h \times h}$ are (block) circulant matrices such that $\mW_l = \mathrm{circ}(\vw_l)$ with $\vw_l$ being the vectorized convolutional filter (a.k.a. kernel). Each row of $\mW_l$ consists of the shifted version of $\vw_l$ with circulant padding. Thus, the size of $\vw_l$ is at most $h$, and usually much smaller than $h$. For simplicity, we consider $\vw_l \in \R^h$ and the corresponding weight perturbation $\vu_l \sim \mathcal{N}(0,\sigma^2 \mR_l)$ with $\mR_l \in \R^{h \times h}$ added onto $\vw_l$.

Given the fact that the (block) circulant matrices can be (block)-diagonalized by discrete Fourier transform (DFT) matrices, i.e., $\mV=\mF$ for 1-dim and $\mV=\mF \otimes \mF$ for 2-dim circular convolution with $\mF$ being DFT matrix, we can design the sensitivity matrices $\mA_l$ in the frequency domain. 

Given the circulant structure of $\mW_l$, we have
\begin{align}
    \mW_l = \mV \Tilde{\mLambda}_l \mV^H
\end{align}
where $\Tilde{\mLambda}_l$ is a diagonal matrix containing eigenvalues of $\mW_l$. When $\mW_l$ is a block circulant matrix (resp. tensor convolutional filters), $\Tilde{\mLambda}_l$ is a block diagonal matrix.

For simplicity, we focus on the scalar case with $\Tilde{\mLambda}_l$ being a diagonal matrix, and the extension to the tensor case can be straightforwardly done according to \cite{yi2022asymptotic}. 
In the frequency domain, the inputs turns to be $\Tilde{\vx}=\mV^H \vx$, and the convolutional filter at the $l$-th layer is $\Tilde{\vw}_l=\mV^H \vw_l \in \R^h$, where $\vw_l$ is the convolutional filter. For circulant matrices, it is readily verified that $\Tilde{\mLambda}_l=\mathrm{diag}(\Tilde{\vw})$. Therefore, we have the output in the frequency domain $\Tilde{f}_{\Tilde{\vw}}(\Tilde{\vx}) = \mV^H f_{\vw}(\vx)$, i.e.,
\begin{align}
    \Tilde{f}_{\Tilde{\vw}}(\Tilde{\vx}) & = \mV^H \mW_d \phi(\mW_{d-1}\phi(\dots \phi(\mW_1 \mV\Tilde{\vx})))\\
    &\approx \Tilde{\mLambda}_d \phi(\Tilde{\mLambda}_{d-1}\phi(\dots \phi(\Tilde{\mLambda}_1 \Tilde{\vx}))),
\end{align}
due to the ReLU activation function $\phi(\cdot)$.
It looks as if we have a feedforward neural network with diagonal weight matrices $\{\Tilde{\mLambda}_l\}_{l=1}^d$, so that the weight perturbation in the frequency domain works as $\Tilde{\vw}_l+\Tilde{\vu}_l$ in a layer-wise manner, where the compact form of weight perturbation in the frequency domain is $\Tilde{\vu}_l = \mV^H \vu_l \in \R^{h}$ due to parameter sharing introduced by circular convolution. 
For the spatial- and frequency-domain neural networks, we have
\begin{align}
    \Norm{\Tilde{f}_{\Tilde{\vw}+\Tilde{\vu}}(\Tilde{\vx})-\Tilde{f}_{\Tilde{\vw}}(\Tilde{\vx})}_2 &= \Norm{\mV^H({f}_{\vw+\vu}({\vx})-{f}_{\vw}({\vx}))}_2 \\ &= \Norm{{f}_{\vw+\vu}({\vx})-{f}_{\vw}({\vx})}_2
\end{align}
because $\mV$ is unitary and the rotation with unitary matrices does not change the $\ell_2$ norm.

In what follows, we consider in a layer-wise manner to figure out proper choices of the sensitivity matrices $\{\mA_l\}_{l=1}^d$ and the corresponding generalization bounds.

\subsubsection{Perturbation bound}
We treat CNNs in the frequency domain as feedforward neural networks with diagonal weight matrices $\{\mLambda_l\}_{l=1}^d$, where the number of weights is reduced from $h^2$ to $h$.
Thus, we have
\begin{align}
\Norm{\Tilde{f}_{\Tilde{\vw}+\Tilde{\vu}}(\Tilde{\vx})-\Tilde{f}_{\Tilde{\vw}}(\Tilde{\vx})}_2^2 &= \Norm{\sum_{l=1}^d \Tilde{\mJ}_l(\Tilde{\vx}) \Tilde{\vu}_l}_2^2+ o(\Norm{\vu_l}_2^2)\\
    &\le d \sum_{l=1}^d \Norm{\Tilde{\mJ}_l(\Tilde{\vx}) \Tilde{\vu}_l}_2^2\\
    &= \sum_{l=1}^d \Norm{\mA_l \vu_l}_2^2
\end{align}
where $\Tilde{\mJ}_l(\Tilde{\vx}) \in \mathbb{C}^{K \times h}$ is the Jacobian matrix in the frequency domain, such that
\begin{align}
    \Tilde{\mJ}_l^H(\Tilde{\vx}) \Tilde{\mJ}_l(\Tilde{\vx}) &\preceq B^2  \frac{\prod_{l=1}^d \Norm{\Tilde{\mLambda}_l}_2^2}{\Norm{\Tilde{\mLambda}_l}_2^2} \mV_{[1:K]} \mV_{[1:K]}^H
\end{align}
and we define
\begin{align}
    \mA_l &= \sqrt{d} B \frac{\prod_{l=1}^d \Norm{\Tilde{\mLambda}_l}_2}{\Norm{\Tilde{\mLambda}_l}_2} \mV_{[1:K]} \mV_{[1:K]}^H \\
    &= \sqrt{d} B \prod_{i\ne l} \Norm{\mV^H \vw_i}_{\infty} \mV_{[1:K]} \mV_{[1:K]}^H
\end{align}
with the last equality is due to the fact that $\Tilde{\mLambda}_i=\mathrm{diag}(\Tilde{\vw})=\mathrm{diag}(\mV^H \vw_l)$,  
so that the spectral norm is the $\ell_{\infty}$ norm of it main diagonal. By such choices of $\{\mA_l\}_{l=1}^d$, the perturbation condition \eqref{eq:perb-cond} is satisfied. 
\subsubsection{The choice of $\sigma$ for a fixed $\hat{\beta}$}
Similarly to the previous cases, we wish to decouple the dependence of $\sigma^2$ on the weights. To this end, for a fixed $\hat{\beta}$, we introduce the approximation of $\mA_l$, i.e., 
\begin{align}
    \hat{\mA}_l = \sqrt{d} B \prod_{i\ne l} \Norm{\mV^H \hat{\vw}_l}_{\infty} \mV_{[1:K]} \mV_{[1:K]}^H
\end{align}
where $\Norm{\mV^H \hat{\vw}_l}_{\infty}=\Norm{\hat{\mW}_l}_2=\hat{\beta}$.
As it holds $\frac{1}{e} \beta^{d-1} \le \hat{\beta}^{d-1} \le e\beta^{d-1}$ when $\Abs{\beta-\hat{\beta}} \le \frac{1}{d}\beta$ \cite{neyshabur2018pac}, we have
\begin{align}
   \frac{1}{e} \mA_l \preceq \hat{\mA}_l \preceq e \mA_l.
\end{align}

Letting
\begin{align} 
    {\mR}_l  = (\mI+\eta^2 \mA_{l}^T \mA_{l} )^{-1}
\end{align}
with $\eta^2=\frac{16   \kappa \Norm{\vw}_2^2}{\gamma^2}$, we have
\begin{align}
    \Tr({\mA}_l {\mR}_l {\mA}_l^T) \le \Tr({\mA}_l {\mA}_l^T) \le e^2\Tr(\hat{\mA}_l \hat{\mA}_l^T)
\end{align}
where the first inequality holds because $\mR_l \preceq \mI$.

Following the same arguments as earlier for the perturbation condition, we need to ensure
\begin{align}
   \sum_{l=1}^d \Norm{{\mA}_l \vu_l}_2^2 &\le \sigma^2 \kappa \sum_{l=1}^d \Tr({\mA}_l {\mR}_l {\mA}_l^T)\\
   &\le e^2 \sigma^2 \kappa \sum_{l=1}^d \Tr(\hat{\mA}_l \hat{\mA}_l^T)
   \le  \frac{\gamma^2}{16}.
\end{align}
To decouple the dependence on the weights, we choose $\sigma^2$ as a function of the approximated $\hat{\mA}_l$, i.e.,
\begin{align}
    \frac{1}{\sigma^2} &= \frac{16e^2 \kappa}{\gamma^2} \sum_{l=1}^d \Tr(\hat{\mA}_l  \hat{\mA}_l^T)\\
    &= \frac{16e^2 \kappa B^2 d}{\gamma^2} \sum_{l=1}^d  \prod_{i\ne l} \Norm{\mV^H \hat{\vw}_l}_{\infty}^2 \label{eq:sigma-upper-circ}
\end{align}
for a fixed $\hat{\beta}$.

\subsubsection{PAC-Bayesian bound for a fixed $\beta$}
For simplicity, let us set $h_l=h$ for all $l$. 
With the optimized $\mR_l$, i.e.,
\begin{align}
    \mR_l^* =  \mV \mathrm{diag}(\underbrace{\frac{1}{1 + \eta^2 \lambda_l^2}, \dots, \frac{1}{1 + \eta^2 \lambda_l^2},}_{K \text{ times}} 1,\dots,1 ) \mV^H
    \label{eq:low-rank-Rl}
\end{align}
with ${\eta}^2=\frac{16 \kappa \Norm{\vw}_2^2}{\gamma^2}$ and $\lambda_l^2=B^2 d \prod_{i\ne l} \Norm{\mV^H \vw_l}_{\infty}^2$. 
Note here that both $\mA_l$ and $\mR_l^*$ are of dimension $h \times h$ due to parameter sharing of convolutional layers. While this does not alter the statistical properties and arguments, we need to consider the compact form of $\vw=(\vw_1,\dots,\vw_d) \in \R^{dh}$ for computing vector/matrix norms, when weight sharing is enabled.

Plugging it into the KL term, together with the choice of $\frac{1}{\sigma^2}$ in \eqref{eq:sigma-upper-circ}, we have
\begin{align}
    \KL & \le \frac{8e^{2} \kappa\Norm{\vw}_2^2}{\gamma^2} \sum_{l=1}^d \Tr(\hat{\mA}_l \hat{\mA}_l^T) \nn \\
    & \qquad \qquad + \frac{1}{2} \sum_{l=1}^d \left(\Tr(\mR_l^*)  - \log \det \mR_l^* - h \right)\\
    &\le \frac{8 e^4 \kappa K \Norm{\vw}_2^2}{\gamma^2 }   \sum_{l=1}^d \lambda_l^2 + \frac{1}{2}\sum_{l=1}^d K \delta({\eta}{\lambda}_l)\\
    &\le \frac{8 (e^4+1) \kappa K \Norm{\vw}_2^2 }{\gamma^2 } \sum_{l=1}^d \lambda_l^2\\
    &\lesssim \gO\left(  \frac{B^2 d^2 K \sum_{l=1}^d \Norm{\vw_l}_2^2}{\gamma^2} \prod_{i\ne l} \Norm{\mV^H \vw_l}_{\infty}^2\right) \\
    &=\gO\left(\frac{B^2d^2K}{\gamma^2} {\Phi}^{\mathrm{circ}}(\vw)\right)
\end{align}
where where $\delta(\cdot)$ is defined in \eqref{eq:delta-func} and ${\Phi}^{\mathrm{circ}}(\vw)=\prod_{i\ne l} \Norm{\mV^H \vw_l}_{\infty}^2 \sum_{l=1}^d \Norm{\vw_l}_2^2$.

Note here that in CNNs ${\Phi}^{\mathrm{circ}}(\vw)$ with convolutional filters $\{\vw_l\}_{l=1}^d$ could be $h$ times smaller than ${\Phi}(\vw)$ with weight matrices $\{\mW_l\}_{l=1}^d$ due the the parameter sharing induced by convolutional layers, i.e., $\mW_l=\mathrm{circ}(\vw_l)$. 
Hence, for a fixed $\hat{\beta}$, the generalization error bound can be given by
\begin{align}
    L_0(f_{\vw}) \le \hat{L}_{\gamma}(f_{\vw}) + \mathcal{O} \left( \sqrt{ \frac{B^2d^2K {\Phi}^{\mathrm{circ}}(\vw) + \ln \frac{m}{\delta}}{\gamma^2m}} \right).
\end{align}

Following the same footsteps as in \cite{neyshabur2018pac} and those in the previous sections, we end up with the final
generalization bound as stated in the theorem for any $\beta$. 

\subsection{Proof of Theorem \ref{thm:toep-matrix}}
\label{proof:toep-matrix}
Let us consider neural networks with certain weight sharing, where $\{\mW_l\}_{l=1}^d \in \R^{h \times h}$ are banded Toeplitz matrices such that $\mW_l = \mathrm{toep}(\vw_l)$ with $\vw_l \in \R^k$ being the Toeplitz kernel. 
Specifically, for a vector $\vv \in \R^k$, $[\mathrm{toep}(\vv)]_{i,j}=v_{j-i}$ for $0 \le j-i \le k-1$ and 0 otherwise. 
Usually $k \ll h$ in practice. 
As Toeplitz matrices are a general case of circulant matrices, the corresponding neural networks typically feature linear convolutional layers. Therefore, each row of $\mW_l$ is actually the shifted Toeplitz kernel with zero padding. The corresponding weight perturbation $\vu_l \sim \mathcal{N}(\mathbf{0},\sigma^2\mR_l)$ with $\mR_l \in \R^{k \times k}$. Denote by $\vw=(\vw_1,\dots,\vw_d)$ and $\vu=(\vu_1,\dots,\vu_d)$ as the concatenated weight and perturbation vectors, respectively.

Similar to the convolution kernels, the Toeplitz structure implies local, translation-invariant sensitivity.
Unlike the circulant case, here we consider the Toeplitz matrix from a spectral perspective, where the perturbation bound does not rely on a Taylor approximation. We still consider the expanded weight matrices $\mW_l = \mathrm{toep}(\vw_l)$ and its corresponding perturbation $\mU_l = \mathrm{toep}(\vu_l)$.
This corresponds to weight sharing in sensitivity, not in the network itself.
The Toeplitz sensitivity matrix enforces correlations among parameter perturbations. That is,
the nearby parameters (with respect to the indices) have correlated influence on the output, and sensitivity depends on relative offsets, rather than the absolute indices.

\subsubsection{Perturbation bound}
Let $\mT \in \R^{h^2 \times h^2}$ be a banded Toeplitz matrix with bandwidth $k$ and the generating sequence $(\{t_i\}_{i=0}^{k-1})$ up to design, and $\psi_{\min} = \min_{\omega} \Abs{\psi(\omega)} >0$ with $\psi(\omega)$ being
the spectral symbol (a.k.a. generating function) of $\mT$, i.e.,
\begin{align}
    \psi(\omega) = \sum_{i=0}^{k-1} t_i e^{-\jmath i\omega} 
\end{align}
with $\omega \in [0,2\pi]$. The eigenvalues of $\mT$ are concentrated around the samples of $\psi(\omega)$. 

Let $\Tilde{\vu}_l=\mathrm{vec}(\mU_l)= \mP \vu_l$, where $\mP \in \R^{h^2 \times k}$ is a structured sparse binary matrix, such that for each column $j$ ($j=0,1,\dots,k-1$), there is a 1 at row index $ih^2+q$ whenever $q-i=j$ with $0 \le i,q < h^2$.
Given that  $\vu_l \sim \mathcal{N}(\mathbf{0},\sigma^2\mR_l)$, it follows that $\Tilde{\vu}_l \sim \mathcal{N}(\mathbf{0}, \sigma^2\mP \mR_l \mP^T)$.
With the properties of Toeplitz matrices, we have
\begin{align} \label{eq:toeplitz-norm}
     \Norm{\mT \Tilde{\vu}_l}_2 &\ge \sigma_{\min}(\mT) \Norm{\Tilde{\vu}_l}_2 \ge \psi_{\min} \Norm{\Tilde{\vu}_l}_2 
\end{align}
where $\sigma_{\min}(\mT)$ is the smallest singular value of $\mT$, and the last inequality is due to the Szeg\"{o} limit theorem \cite{gray2006toeplitz}. 

Define an auxiliary matrix $\Tilde{\mA}_l \in \R^{h^2 \times h^2}$ be a Toeplitz matrix with
\begin{align}
        \Tilde{\mA}_l =  \frac{ e\sqrt{d}B  \prod_l \Norm{\mW_l}_2}{\psi_{\min}\Norm{\mW_l}_2}\mT, \ \forall~l,
    \end{align}
so that we have
\begin{align}
    \sum_{l=1}^d\Norm{\Tilde{\mA}_l \Tilde{\vu}_l}_2^2 &= \frac{e^2B^2d\prod_l \Norm{\mW_l}_2^2}{\psi_{\min}^2} 
    \sum_{l=1}^d \frac{\Norm{\mT \Tilde{\vu}_l}_2^2}{\Norm{\mW_l}_2^2} \\
    &\ge e^2B^2 d\prod_l \Norm{\mW_l}_2^2  \sum_{l=1}^d \frac{\Norm{\mU_l}_F^2}{\Norm{\mW_l}_2^2} \\
    &\ge e^2B^2 \prod_l \Norm{\mW_l}_2^2 \left(\sum_{l=1}^d \frac{\Norm{\mU_l}_2}{\Norm{\mW_l}_2}\right)^2 \\
    &\ge \Norm{f_{\vw+\vu}(\vx)-f_{\vw}(\vx)}_2^2
\end{align}
where the first inequality is due to \eqref{eq:toeplitz-norm} and $\Norm{\Tilde{\vu}_l}_2^2=\Norm{\mU_l}_F^2$. The rest follows similarly to the diagonal case, where the last inequality holds regardless of the structure of $\mW_l$. In fact, the Toeplitz case yields a looser perturbation bound than the diagonal case, yet its underlying weight sharing would lead to a tighter generalization bound.

Note here that we could either apply the concentration inequality to $\Norm{\Tilde{\mA}_l \Tilde{\vu}_l}_2$ with $\Tilde{\vu}_l \sim \mathcal{N}(\mathbf{0}, \sigma^2\mP \mR_l \mP^T)$ containing repeated elements in $\vu_l$ due to the Toeplitz structure, or to $\Norm{{\mA}_l {\vu}_l}_2$ with $\vu_l \sim \mathcal{N}(\mathbf{0},\sigma^2\mR_l)$ directly with a refined sensitivity matrix $\mA_l$, such that
\begin{align}
    \mA_l = \Tilde{\mA}_l \mP.
\end{align}
Since
\begin{align}
    \Tilde{\vu}_l^T \Tilde{\mA}_l^T \Tilde{\mA}_l \Tilde{\vu}_l = \vu_l^T \mA_l^T \mA_l \vu_l,
\end{align}
we have the perturbation bound holds, i.e.,
\begin{align}
    \sum_{l=1}^d\Norm{{\mA}_l {\vu}_l}_2^2 \ge \Norm{f_{\vw+\vu}(\vx)-f_{\vw}(\vx)}_{\infty}^2
\end{align}
with a $h^2 \times k$ sensitivity matrix $\mA_l$.
In doing so, the concentration inequality of $\Norm{\mA_l \vu_l}_2$ follows as in the previous cases without loss of rigor.
This demonstrates that the perturbation bounds do not actually require Jacobian domination, but energy domination in parameter space.

\subsubsection{The choice of $\sigma$ for a fixed $\hat{\beta}$}
Similarly to the diagonal case, for a fixed $\hat{\beta}$, we introduce the approximation of $\mA_l$, i.e., 
\begin{align}
        \hat{\mA}_l =  \frac{ e\sqrt{d}B \prod_l \Norm{\hat{\mW}_l}_2}{\psi_{\min}\Norm{\hat{\mW}_l}_2}\mT \mP, \ \forall~l.
\end{align}
Then it follows
\begin{align}
   \frac{1}{e^2} \mA_l \mA_l^T \preceq \hat{\mA}_l \hat{\mA}_l^T \preceq e^2 \mA_l \mA_l^T.
\end{align}
By enforcing
\begin{align}
   \sum_{l=1}^d \Norm{{\mA}_l {\vu}_l}_2^2 &\le \sigma^2 \kappa \sum_{l=1}^d \Tr({\mA}_l {\mR}_l {\mA}_l^T)\\
   &\le e^2 \sigma^2 \kappa \sum_{l=1}^d \Tr(\hat{\mA}_l \hat{\mA}_l^T)
   \le  \frac{\gamma^2}{16}
\end{align}
with ${\mR}_l \preceq \mI$, we set
\begin{align}
    \frac{1}{{\sigma}^2} &= \frac{16e^2 \kappa}{\gamma^2} \sum_{l=1}^d \Tr(\hat{\mA}_l \hat{\mA}_l^T).
\end{align}

\subsubsection{PAC-Bayesian bound for a fixed $\hat{\beta}$}
Due to weight sharing, the weights of interest in the KL term reduce to $\vw=(\vw_1,\vw_2,\dots,\vw_d)$.
Next, according to \eqref{eq:opt-Rl}, we have the optimized ${\mR}_l$, i.e.,
\begin{align} \label{eq:toeplitz-Rl}
    {\mR}_l^* = (\mI+{\eta}^2 \mA_l^T\mA_l )^{-1}
    = (\mI+{\eta}^2 {\lambda}_l^2 \mP^T \mT^T  \mT \mP)^{-1}
\end{align}
where $\eta^2=\frac{16   \kappa \Norm{\vw}_2^2}{\gamma^2}$ and 
\begin{align}
    \lambda_l = \frac{ e\sqrt{d}B  \prod_l \Norm{\mW_l}_2}{\psi_{\min}\Norm{\mW_l}_2}.
\end{align}

To make the optimized ${\mR}_l^*$ tractable, we consider the asymptotic analysis. 
Due to the Szeg\"{o} limit theorem with respect to  the eigenvalue distribution, where $g(\mT) \leftrightarrow g(\psi(\omega))$ for any continuous function $g$, in the sense that \cite{gray2006toeplitz}
\begin{align}
    \lim_{h^2 \to \infty} \frac{1}{h^2}\sum_{j=1}^{h^2} g(\sigma_j(\mT)) = \frac{1}{2\pi} \int_{0}^{2\pi} g(\psi(\omega)) \mathrm{d}\omega
\end{align}
we have, as $h^2 \to \infty$, that
\begin{align}
\Tr({\mA}_l {\mA}_l^T) &\sim \frac{k}{2\pi}\int_{0}^{2\pi} \lambda_l^2 \Abs{\psi(\omega)}^2 \mathrm{d} \omega  \\
    \Tr(\mR_l^*)  - \log \det \mR_l^* - \dim(\mR_l^*) &\sim \frac{k}{2\pi}\int_{0}^{2\pi}  \delta({\eta}{\lambda}_l \Abs{\psi(\omega)}) \mathrm{d} \omega\\
    &\le \frac{k}{2\pi}\int_{0}^{2\pi}  {\eta}^2{\lambda}_l^2 \Abs{\psi(\omega)}^2 \mathrm{d} \omega
\end{align}
with $\delta(x)=\log(1+x^2)-\frac{x^2}{1+x^2} \le x^2$ as defined before.

Through $\mR_l$, Toeplitz suppresses structured directions, keeps sensitivity only where the symbol is large, corresponding to smoothness in parameter sensitivity.
With the choice of $\sigma^2$, the KL divergence can be upper bounded by
\begin{align}
    \KL & \le \frac{8e^{2} \kappa\Norm{\vw}_2^2}{\gamma^2} \sum_{l=1}^d \Tr(\hat{\mA}_l \hat{\mA}_l^T) \nn \\
    &  + \frac{1}{2} \sum_{l=1}^d \left(\Tr(\mR_l^*)  - \log \det \mR_l^* - \dim(\mR_l^*) \right)\\
    &\lesssim \frac{8 e^4 \kappa  \Norm{\vw}_2^2}{\gamma^2 }   \sum_{l=1}^d \frac{k}{2\pi} \int_{0}^{2\pi} \lambda_l^2 \Abs{\psi(\omega)}^2 \mathrm{d} \omega \nn \\
    & \qquad \qquad \qquad + \frac{1}{2}\sum_{l=1}^d \frac{k}{2\pi} \int_{0}^{2\pi} \delta({\eta}{\lambda}_l \Abs{\psi(\omega)}) \mathrm{d} \omega\\
    &\lesssim \frac{8 (e^4+1) \kappa  \Norm{\vw}_2^2 }{\gamma^2 } \sum_{l=1}^d \frac{k}{2\pi} \int_{0}^{2\pi} \lambda_l^2 \Abs{\psi(\omega)}^2 \mathrm{d} \omega\\
    &\lesssim \gO\left(  \frac{B^2 d^2 k \beta^{2d-2}}{\psi_{\min}^2 \gamma^2} \sum_{l=1}^d \Norm{\vw_l}_2^2 \int_{0}^{2\pi} \Abs{\psi(\omega)}^2 \mathrm{d} \omega\right) \\
    &\lesssim \gO\left(\frac{B^2d^2k}{\gamma^2} \frac{\psi_{\max}^2}{\psi_{\min}^2} \prod_{i\ne l} \Norm{\vw_i}_1^2 \sum_{l=1}^d \Norm{\vw_l}_2^2\right)\\
    &=\gO\left(\frac{B^2d^2k}{\gamma^2} \Phi^{\mathrm{toep}}(\vw)\right)
\end{align}
where $\psi_{\max}=\max_{\omega}\Abs{\psi(\omega)}$, $\Norm{\vw_l}_2 \le \beta=\Norm{\mW_l}_2 \le \Norm{\vw_l}_1$ for Toeplitz matrices, and the spectral complexity becomes
\begin{align}
    \Phi^{\mathrm{toep}}(\vw)=\frac{\psi_{\max}^2}{\psi_{\min}^2} \prod_{i\ne l} \Norm{\vw_i}_1^2 \sum_{l=1}^d \Norm{\vw_l}_2^2. 
\end{align}

With the same arguments as the previous, we end up with the final generalization bound in the theorem 
\begin{align}
    L_0(f_{\vw}) \le \hat{L}_{\gamma}(f_{\vw}) + \mathcal{O} \left( \sqrt{ \frac{B^2d^2k \Phi^{\mathrm{toep}}(\vw) + \ln \frac{md}{\delta}}{\gamma^2m}} \right)
\end{align}
for any ${\beta}$. As the ``kernel'' spectral complexity $\Phi^{\mathrm{toep}}(\vw)$ is usually much smaller than that of the spectral complexity of entire weight matrices, this bound is tighter than the previous one \cite{neyshabur2018pac}, when network structures follow the Toeplitz style.

In what follows, we consider some special designs of $\mT$.

\textbf{Orthogonal circulant matrix.} Let $\mT$ be an orthogonal circulant matrix, i.e., $\mT^T \mT = \mI $, which is also Toeplitz. So, $\psi_{\max}=\psi_{\min}=1$. 
As such, the scaling factor reduces to 
\begin{align}
    \gO\left(\sqrt{\frac{B^2d^2k}{\gamma^2} \prod_{i\ne l} \Norm{\vw_i}_1^2 \sum_{l=1}^d \Norm{\vw_l}_2^2}\right)
\end{align}
which is comparable with that in Theorem \ref{thm:circulant-matrix}, because $\Norm{\vw_l}_2 \le \Norm{\mV^H \vw_l}_{\infty} \le \Norm{\vw_l}_1$, yet is tighter than others.

\textbf{Geometric kernel with exponential decay.}
Let $\mT$ be a banded symmetric Toeplitz matrix with exponential decay, i.e., $t_i=\rho^{\Abs{i}}$, with $\rho \in (0,1)$.
The exponentially decaying entries encode smooth attenuation of perturbation effects, mimicking spatial or channel locality. The design avoids unnecessary coupling of unrelated coordinates (unlike full matrices) while still being richer than diagonal.

In this case, we have
\begin{align}
    \psi(\omega)=\frac{1-\rho^2}{1-2\rho \cos\omega + \rho^2}
\end{align}
which yields
\begin{align}
    \frac{\psi_{\max}}{\psi_{\min}} = \frac{1+\rho}{1-\rho}.
\end{align}
The diagonal case is a special Toeplitz when $\rho=0$, where the spectral symbol is flat, and all frequencies are penalized equally. The symbol $\psi(\omega)$ is a frequency filter, where the low-frequency components are penalized more. Thus, the resulting scaling factor with $\rho=0$ reduces to 
\begin{align}
    \gO\left(\sqrt{\frac{B^2d^2k}{\gamma^2} \prod_{i\ne l} \Norm{\vw_i}_1^2 \sum_{l=1}^d \Norm{\vw_l}_2^2}\right)
\end{align}
which is strictly tighter than those in Theorem \ref{thm:diag-matrix} and in \cite{neyshabur2018pac}.


\bibliographystyle{IEEEtran}

\end{document}